\documentclass[11pt]{article}
\usepackage{url}            
\usepackage{booktabs}       
\usepackage{amsfonts}       
\usepackage{nicefrac}       
\usepackage{microtype}      

\usepackage{times}

\usepackage{hyperref}
\usepackage{tabularx}
\usepackage{soul}
\usepackage{xcolor}
\usepackage{color}

\usepackage{fontenc}
\usepackage{enumerate}
\usepackage{enumitem} 

\usepackage{graphicx,epsfig,url,subcaption} \graphicspath{ {figures/} }
\usepackage{setspace}
\usepackage[normalem]{ulem} 

\usepackage{amsmath,amsthm,amssymb} 
\usepackage{mathtools}
\usepackage{bm,bbm} 			
\usepackage{algorithm,algpseudocode} 
 
\usepackage[top=1in, bottom=1in, left=1.1in, right=1.1in]{geometry} 

\usepackage[backend=bibtex,maxnames=10,style=alphabetic,sorting=none]{biblatex}

\usepackage{localdefs}	    
\bibliography{references} 

\usepackage{authblk}
\title{Continual Mean Estimation Under User-Level Privacy}

\author[1]{
	Anand Jerry George\textsuperscript{ \textasteriskcentered} 
}
\author[2]{ 
	Lekshmi Ramesh\textsuperscript{ \textdagger} 
}
\author[2]{
	Aditya Vikram Singh\textsuperscript{ \textdaggerdbl}
}
\author[2]{
	Himanshu Tyagi\textsuperscript{ \textparagraph} 
}
\affil[1]{\'Ecole Polytechnique F\'ed\'erale de Lausanne}
\affil[2]{Indian Institute of Science, Bangalore}
\affil[ ]{ 
	\textsuperscript{\textasteriskcentered}\email{anand.george@epfl.ch}, \{\textsuperscript{\textdagger}\emailiisc{lekshmi}, \textsuperscript{\textdaggerdbl}\emailiisc{adityavs}, \textsuperscript{\textparagraph}\emailiisc{htyagi}\}@iisc.ac.in 
}
\date{}

\begin{document}
\maketitle
	
\begin{abstract}
	We consider the problem of continually releasing an estimate of the population mean of a stream of samples that is user-level differentially private (DP). At each time instant, a user contributes a sample, and the users can arrive in arbitrary order.
	Until now these requirements of continual release and user-level privacy were considered in isolation.
	But, in practice, both these requirements come together as the users often contribute data repeatedly and multiple queries are made.
 We provide an algorithm that 
 outputs a mean estimate at every time instant $t$ such that the overall release is user-level $\varepsilon$-DP and has the following error guarantee: Denoting by
 $M_t$ the maximum number of samples contributed by a user, as long as 
 $\tilde{\Omega}(1/\varepsilon)$ users have $M_t/2$ samples each, the error at time $t$ is
 $\tilde{O}(1/\sqrt{t}+\sqrt{M}_t/t\varepsilon)$. This is a universal error guarantee which is valid for all arrival patterns of the users. Furthermore, it (almost) matches
 the existing lower bounds for the single-release setting at all time instants when users have contributed equal number of samples. 
\end{abstract}
\section{Introduction}\label{sec:introduction}
Aggregate queries over data sets were originally believed to maintain the privacy of data contributors. However, over the past two decades several attacks have been proposed to manipulate the output of aggregate queries to get more information about an individual user's data~\cite{Narayanan_SSP_2008},~\cite{Sweeney_2013},~\cite{Garfinkel_ACM_2019}. To address this, several mechanisms have been proposed to release a noisy output instead of the original query output. Remarkably, these mechanisms have been shown to preserve privacy under a mathematically rigorous privacy requirement condition called {\em differential privacy}~\cite{Dwork_TCC_2006}. But, until recently, the analysis assumed that each user contributes one data-point and not multiple. Furthermore, the data set was assumed to be static and only one query was made. Our goal in this paper is to address a more practical situation where multiple queries are made and the data set keeps on getting updated between two queries, using contributions from existing or new users. We provide an almost optimal private mechanism for this case for the specific running-average query, which can easily be adopted to answer more general aggregate queries as well.

\subsection{Problem formulation and some heuristics}
Consider a stream $(x_{1},\ldots,x_{T})$ of $T$ data points contributed by $n$ users, where $T\geq n$. 
For simplicity, we assume that one point is contributed at each time: $x_t\in \mR^d$ is contributed by a user $u_t\in[n]$ at time $t$. The maximum number of samples contributed by any user till time instant $t$ is denoted by $M_{t}$, and we assume that $M_t\le m$, i.e., each user contributes at most $m$ samples.
\medskip

We formulate our query release problem as a statistical estimation problem to identify an optimal mechanism. Specifically, we assume that each $x_{t}$ is drawn independently from a distribution $P$ on $\mathbb{R}^{d}$ with unknown mean $\mu$. 
At each time step $t\in[T]$, we are required to output an estimate $\est_{t}$ for the mean of $P$, while guaranteeing that the sequence of outputs $(\est_{1},\ldots,\est_{t})$ is {\em user-level $\varepsilon$-differentially private ($\varepsilon$-DP)}. Namely, the output is $\varepsilon$-DP with respect to the user input comprising all the data points contributed by a single user~\cite{Amin_ICML_2019}.
\medskip

A naive application of standard differentially private mechanisms at each time step will lead to error rates with suboptimal dependence on the time window and the number of samples contributed by each user. For instance, consider the most basic setting where users draw samples independently from  a Bernoulli distribution with parameter $\mu$. At any time $t\in[T]$, a single user can affect the sample mean $(1/t)\sum_{i=1}^{t}x_{i}$ by at most $M_{t}/t$. Adding Laplace noise with parameter 
$M_{t}/(t\varepsilon)$ to this running sum will therefore guarantee user-level $\varepsilon$-DP at each step (implying $\varepsilon t$-DP over $t$ steps) and an error of $O(1/\sqrt{t}+M_{t}/t\varepsilon)$. Rescaling the privacy parameter, we get an error that scales as $O(1/\sqrt{t}+M_{t}/\varepsilon)$. 
While the statistical error term is optimal, the error term due to privacy is not -- it does not, in fact, improve as time progresses.
One can do better by using mechanisms specific to the streaming setting such as the binary mechanism~\cite{Chan_ICALP_2010},~\cite{Dwork_STOC_2010}, as we describe in more detail in Section \ref{sec:algorithm}. Indeed, we will see that the privacy error term can be improved in two aspects: first, it can be made to decay as time progresses, and second, it only needs to grow sublinearly with $M_{t}$. 
\medskip 

However, a key challenge still remains. Each user is contributing multiple samples to the stream, and different samples from the same user can come at arbitrary time instants. 
The output must depend on the the arrival pattern of the users.
For instance, when all samples in the stream are contributed by a single user,
we cannot release much information. Indeed,
changing this user's data can potentially change any query responses by a large amount, leading to increased sensitivity and addition of large amounts of noise to guarantee user-level privacy. A better strategy is to withhold answers for a while until a new user arrives -- this provides a sort of {\em diversity advantage} and reduces the amount of noise we need to add. The process of withholding alone does not, however, lead to an optimal error rate.  We additionally need to control sensitivity by forming local averages of each users samples and then truncating these averages, as is done in the one-shot setting~\cite{Levy_Neurips_2021}, before adding noise.

\subsection{Summary of contributions}
We present a modular and easy-to-implement algorithm for continual mean estimation under user-level privacy constraints. 
Our main algorithm has two key components: a sequence of binary mechanisms that keep track of truncated averages 
 and a withhold-release mechanism that decides when to update the mean estimate. 
The role of binary mechanisms is similar to~\cite{Chan_ICALP_2010},~\cite{Dwork_FOCS_2010} -- to minimize the number of outputs each data point influences. 
The withhold-release mechanism releases a query only when there is ``sufficient diversity,'' i.e., enough users have contributed to the data. In fact, in a practical implementation of our algorithm, we will need to maintain this diversity by omitting excessive number of samples from a few users from the mean estimate evaluation.
Together, these components allow us to balance the need for more data for accuracy and the need for 
controlling sensitivity due to a single user's data. 
\medskip
 
The resulting performance is characterized roughly as follows; 
see Section \ref{sec:algorithm} for the formal statement.
\begin{mainresult}
Our algorithm provides a user-level $\varepsilon$-DP mechanism to output
a mean estimate at every time instant $t$ when the data has ``sufficient diversity''
with error $\tilde{O}(1/\sqrt{t}+\sqrt{M_{t}}/t\varepsilon)$, where
$M_{t}$ is the maximum number of samples contributed by any user till time step $t$. 
\end{mainresult}

We do not make any assumptions on the order in which the data points arrive from different users. The ``sufficient diversity'' condition in our theorem (formalized in Definition~\ref{def:diversity}) codifies the necessity to have sufficient number of users with sufficient number of samples for our algorithm to achieve good accuracy while ensuring user-level privacy. Section~\ref{sec:algorithm} further elaborates the difficulty of obtaining good accuracy with arbitrary user ordering and how our \textit{exponential withhold-release} mechanism helps overcome this. 

\subsection{Prior Work} \label{subsec:priorwork}
Continually releasing query answers can leak more information when compared to the single-release setting, since each data point can now be involved in multiple query answers.
In addition to this, each user can contribute multiple data points, in which case we would like to ensure that the output of any privacy-preserving mechanism we employ remains roughly the same even when \emph{all} of the data points contributed by a user are changed.
There have been a few foundational works on both these fronts~\cite{Dwork_STOC_2010},~\cite{Jain_Arxiv_2021},~\cite{Levy_Neurips_2021}, but a unified treatment is lacking. Indeed, addressing the problem of user-level privacy in the streaming setting was noted as an interesting open problem in \cite{Nikolov_2013}. 
\smallskip

Privately answering count queries in the continual release setting was first addressed in ~\cite{Dwork_STOC_2010},~\cite{Chan_ICALP_2010}, where the binary mechanism was introduced. This mechanism releases an estimate for the number of ones in a bit stream by using a small number of noisy partial sums. Compared to the naive scheme of adding noise at every time step which leads to an error that grows linearly with $T$, the binary tree mechanism has error that depends logarithmically on $T$.
Following this, other works have tried to explore (item-level) private continual release in other settings~\cite{Chan_PETS_2012},~\cite{Bolot_ICDT_2013},~\cite{Ding_neurips_2017},~\cite{Joseph_neurips_2018},~\cite{Perrier_NDSS_2019}, ~\cite{Jain_Arxiv_2021}.
\medskip

Releasing statistics under user-level privacy constraint was discussed in \cite{Dwork_STOC_2010},~\cite{Dwork_ICS_2010}, and  has recently begun to gain attention~\cite{Levy_Neurips_2021},~\cite{Cummings_neurips_2021},~\cite{Narayanan_ICML_2022}. In particular, \cite{Levy_Neurips_2021} considers user-level privacy for one-shot $d$-dimensional mean estimation and shows that a truncation based estimator achieves error that scales as $\tilde{O}(\sqrt{d/mn}+\sqrt{d}/\sqrt{m}n\varepsilon)$, provided there are at least $\tilde{O}(\sqrt{d}/\varepsilon)$ users. This requirement on the number of users was later improved to $\tilde{O}(1/\varepsilon)$ in \cite{Narayanan_ICML_2022}. \cite{Cummings_neurips_2021} takes into account heterogeneity of user distributions for mean estimation under user-level privacy constraints.

\subsection{Organization}
Section \ref{sec:preliminaries} describes the problem setup and gives a brief recap on user-level privacy and the binary mechanism. We describe our algorithm in Section \ref{sec:algorithm} and give a rough sketch of its privacy and utility guarantees. Section \ref{sec:extensions} includes a discussion on extension to $d$-dimensional mean estimation and handling other distribution families. Finally, in Section \ref{sec:lowerbound} we discuss the optimality of our algorithm, and end with discussion on future directions in Section \ref{sec:discussion}.

\section{Preliminaries}\label{sec:preliminaries}

\subsection{Problem setup}
We observe an input stream of the form $((x_1,u_1),(x_2,u_2),\ldots,(x_T,u_T))$, where $x_{t}\in\cX$ is the sample and $u_{t}\in[n]$ is the user contributing the sample $x_{t}$. The samples $\paren{x_t}_{t\in[T}$ are drawn independently from a distribution with unknown mean. The goal is to output an estimate $\est_t$ of the mean for every $t \in [T]$, such that the overall output $(\est_1,\ldots,\est_T)$ is user-level $\eps$-DP. We present our main theorems for the case when each $x_{t}$ is a Bernoulli random variable with unknown mean $\mean$. Extensions to other distribution families are discussed in Section~\ref{sec:extensions}.

\subsection{Differential privacy}

Let $\sigma =\Big((x_{t},u_{t})\Big)_{t\in[T]}$ and $\sigma^{\prime}=\Big((x_{t}^{\prime},u_{t}^{\prime})\Big)_{t\in[T]}$ denote two streams of inputs.
We say that $\sigma$ and $\sigma^{\prime}$ are \textit{user-level neighbors} if there exists $j\in[n]$ such that $x_{t}=x_{t}^{\prime}$ for every $t\in[T]$ satisfying $u_{t}\ne j$.
We now define the notion of a user-level $\varepsilon$-DP algorithm.
\begin{definition}
	An algorithm $\mathcal{A}:(\mathcal{X} \times [n])^{T}\rightarrow \mathcal{Y}$ is said to be user-level  $\varepsilon$-DP if for every pair of streams $\sigma, \sigma^{\prime}$ that are user-level neighbors and every subset $Y\subseteq \mathcal{Y}$,
		$\mathbb{P}(\mathcal{A}(\sigma)\in Y)\le e^{\varepsilon} \ \mathbb{P}(\mathcal{A}(\sigma^{\prime})\in Y).$
\end{definition}
We will be using the following composition result satisfied by DP mechanisms.

\begin{lemma}\label{lem:composition}
Let $\cM_{i}:(\mathcal{X} \times [n])^{T}\rightarrow \cY$ be user-level $\varepsilon_i$-DP mechanisms for $i\in[k]$. Then the composition $\cM:(\mathcal{X} \times [n])^{T}\rightarrow\mR^{k}$ of these mechanisms given by $\cM(x):=(\cM_{1}(x),\ldots,\cM_{k}(x))$ is user-level $(\sum_{i=1}^{k}\varepsilon_{i})$-DP.
\end{lemma}

We now define the Laplace mechanism, 
a basic privacy primitive we employ. We use ${\rm Lap}(b)$ to denote the Laplace distribution with parameter $b$.
\begin{definition}
For a function $f:\mathcal{X}^{n}\rightarrow\mR^k$, the Laplace mechanism with parameter $\varepsilon$ is a randomized algorithm $\cM:\cX^{n}\rightarrow \mR^k$ with
$	\cM(x) = f(x) + 
	(
	Z_1,
	\cdots,
	Z_k
	)$,
where $Z_1,\ldots,Z_k  \overset{\rm i.i.d.}{\sim} {\rm Lap}(\Delta f/\varepsilon)$, $\Delta f:= \max_{x,x^{\prime}\in\cX^n}\norm{f(x)-f(x^{\prime})}_1$ and addition is coordinate-wise.
\end{definition} 

The following lemma states the privacy-utility guarantee of the Laplace mechanism. 
\begin{lemma}\label{lem:lap_mech}
	The Laplace mechanism is $\varepsilon$-DP and guarantees  
	\[
	\mathbb{P}\Big({\norm{f(x)-\cM(x)}_\infty \ge \Delta f/\varepsilon\cdot\ln(k/\delta)}\Big)\le\delta  \quad \forall\,\delta\in(0,1].
	\]
\end{lemma}

\subsection{Binary mechanism} \label{subsec:binmech}

The binary mechanism (Algorithm \ref{alg:binmech}) \cite{Chan_ICALP_2010} receives as input a stream $(x_1,x_2,\ldots,x_T)$ and outputs a noisy sum $S_t$ of stream up to time $t$ 
at every $t \in [T]$ such that the overall output $(S_1,S_2,\ldots,S_T)$ is $\eps$-DP. Furthermore, for any $t \in [T]$, the error between the output and the running-sum satisfies $\abs{S_t - \sum_{i=1}^{t}x_i} = O\paren{ \frac{\Delta}{\eps} \log T \sqrt{\log t} \ln \frac{1}{\delta} }$ with probability at least $1-\delta$. Here, $\Delta$ is the magnitude of maximum possible variation of an element in the stream; for e.g., if $1 \leq x_i \leq 3$ for every $i \in [T]$, then $\Delta = 2$.
\medskip


We observe two important properties of the binary mechanism (Algorithm \ref{alg:binmech}):
\begin{itemize}
	\item[(i)] Any stream element $x_i$ is involved in computing at most $(1+\log T)$ terms of the array ${\rm NoisyPartialSums}$. For e.g., $x_1$ is needed only while computing ${\rm NoisyPartialSums}[t]$ for $t = 1,2,4,8,$ and so on. 
	\item[(ii)] For any $t$, the output $S_t$ is a sum of at most $(1+\log t)$ terms from the array ${\rm NoisyPartialSums}$.
\end{itemize}
We now describe how these properties of the binary mechanism lead to privacy and utility (accuracy) guarantees.
\medskip

\textbf{Privacy.} Since, for every $t$, the output $S_t$  is a function of ${\rm NoisyPartialSums}$, it suffices to ensure that the array ${\rm NoisyPartialSums}$ is $\eps$-DP. From (i) above, changing a stream element will change at most $(1+\log T)$ terms of the array ${\rm NoisyPartialSums}$. Moreover, since the maximum variation of a stream element is $\Delta$, each of these terms of ${\rm NoisyPartialSums}$ will change by at most $\Delta$. Overall, changing an arbitrary stream element $x_i$ changes the $\ell_1$-norm of the array ${\rm NoisyPartialSums}$ by at most $\Delta(1+\log T)$. Thus, to ensure that the overall stream of outputs from the binary mechanism is $\eps$-DP, it suffices to add ${\rm Lap}(\eta)$ noise to each term of ${\rm NoisyPartialSums}$, where $\eta = \Delta(1+\log T)/\eps$.
\medskip

\textbf{Utility.} From (ii) above, since the output $S_t$ is a sum of at most $(1+\log t)$ terms from the array ${\rm NoisyPartialSums}$, where each term of ${\rm NoisyPartialSums}$ has an independent ${\rm Lap}(\eta)$ noise, we have that, with probability at least $1-\delta$, $\abs{S_t - \sum_{i=1}^{t}x_i} \leq \eta \sqrt{1+\log t} \ln\frac{1}{\delta}$.

\begin{algorithm}[ht]
	\caption{Binary Mechanism \cite{Chan_ICALP_2010}}
	\label{alg:binmech}
	\begin{algorithmic}[1]
		\Require $(x_t)_{t\in[T]}$ (stream), $T$ (stream length), $\Delta$ (max variation of a stream element), $\eps$ (privacy parameter)
		
		\State {\bf Initialize} ${\rm Stream}$, ${\rm NoisyPartialSums}$ (arrays of length $T$)
		
		\State $\eta \gets {\Delta (1+\log T)}/{\eps}$ \Comment \textit{noise parameter}
		
		\For{$t = 1,2,\ldots,T$}
		
		\State ${\rm Stream}[t] \gets x_t$
		
		\State Express $t$ in binary form: $t = \sum_{j=0}^{\lfloor \log t \rfloor} (b_j \cdot 2^j)$ \Statex \Comment \textit{$b_j \in \set{0,1}$}
		\State $j^{\ast} \gets \min\set{j: b_j \neq 0}$ \Comment \textit{e.g. $j^{\ast}=0$ if $t$ odd}
		
		\State ${\rm NoisyPartialSums}[t] \gets$ \Statex \hspace{24mm} $\paren{\sum_{i=t-2^{j^{\ast}}+1}^{t} {\rm Stream}[i]} + \ {\rm Lap}(\eta)$ \Statex \Comment \textit{noisy sum of the latest $2^{j^\ast}$ elements of ${\rm Stream}$}
		
		\State ${\rm Sum} \gets 0$, ${\rm index} \gets 0$	
		\For{$j=\lfloor \log t \rfloor, \lfloor \log t \rfloor -1, \ldots,0$}
			\State ${\rm index} \gets {\rm index} + (b_j \cdot 2^j)$
			\State ${\rm Sum} \gets {\rm Sum} + {\rm NoisyPartialSums}[{\rm index}]$
		\EndFor
		\State {\bf return} $S_t = {\rm Sum}$
		\EndFor
	\end{algorithmic}
\end{algorithm}

In our algorithms, we invoke an instance of the binary mechanism as $\binmech$, and abstract out its functionality using the following three ingredients:
\begin{itemize} \itemsep0em
	\item $\binmechstream$: This is an array which will act as the input stream for the binary mechanism $\binmech$. In our algorithms, we will feed an element to $\binmechstream$ only at certain special time instances.
	\item $\binmechpartial$: This array will be of the same length as $\binmechstream$. The $k$-th term of $\binmechpartial$ will be computed after the $k$-th element enters $\binmechstream$. The magnitude of Laplace noise to be added while computing terms of $\binmechpartial$ (magnitude of Laplace noise would be same for all terms) will be passed as a noise parameter while invoking $\binmech$. The sum output by $\binmech$ would be formed by combining terms from $\binmechpartial$. 
	\item $\binmechsum$: Suppose, at time $t$, $\binmechstream$ (and, thus, $\binmechpartial$) has $k$ elements. Then, $\binmechsum$ is a function which, when invoked at time $t$, will output $S_k$ (noisy sum of all $k$ elements in $\binmechstream$) by combining terms stored in $\binmechpartial$.
	
\end{itemize}

\section{ALGORITHM}\label{sec:algorithm}
In this section, we build up to our main algorithm (Algorithm \ref{alg:multcounterfull}), pointing out the main ideas along the way.

\paragraph{A naive use of binary mechanism.} Given a stream $\big((x_t,u_t)\big)_{t\in[T]}$ (where $x_t \overset{\rm i.i.d.}{\sim} \ber(\mean)$), Algorithm~\ref{alg:naive} presents a straightforward way to estimate mean $\mean$ in the continual-release user-level DP setting. The algorithm feeds $x_t$ to binary mechanism $\binmech$ at every time instant $t$. Since each user contributes at most $m$ samples, changing a user will change at most $m(1+\log T)$ terms of $\binmechpartial$. Moreover, since the maximum variation of any stream element $x_t$ is $1$, changing a user will change the $\ell_1$-norm of $\binmechpartial$ by at most $m(1+\log T)$. Thus, for $\eta = \frac{m(1+\log T)}{\eps}$, adding ${\rm Lap}(\eta)$ noise to each term of $\binmechpartial$ ensures that this algorithm is user-level $\eps$-DP. To obtain an accuracy guarantee at any given time $t$, we observe the following: Since $x_1,\ldots,x_t$ are independent $\ber(\mean)$ samples, we have that, with probability at least $1-\delta$, $\abs{\sum_{i=1}^{t}x_i - t\mean} \leq \sqrt{\frac{t}{2} \ln\frac{2}{\delta}}$. Moreover, since $S_t$ is a sum of at most $(1+\log t)$ terms from $\binmechpartial$, where each term of $\binmechpartial$ has an independent ${\rm Lap}(\eta)$ noise, we have that, with probability at least $1-\delta$, $\abs{S_t - \sum_{i=1}^{t}x_i} \leq \eta \sqrt{1+\log t} \ln\frac{1}{\delta}$. Thus, using union bound, we have with probability at least $1-2\delta$ that
\begin{equation} \label{eq:naive}
\abs{\est_t-\mean } = O\paren{ \sqrt{\frac{1}{t} \ln\frac{1}{\delta}} + \frac{m}{t\eps} \sqrt{\log t} \log T \ln\frac{1}{\delta} }.
\end{equation}
The first term on the right hand side of \eqref{eq:naive} is the usual statistical error that results from using $t$ samples to estimate $\mu$. The second term in the error (``\textit{privacy error}''), however, results from ensuring that the stream of estimates $(\est_1,\ldots,\est_T)$ is user-level $\eps$-DP; note that this term is linear in $m$. We next describe a ``wishful'' scenario where the privacy error can be made proportional $\sqrt{m}$. We will then see how to obtain such a result in the general scenario, which is our main contribution.

\begin{algorithm}[ht]
	\caption{Continual mean estimation (naive)}
	\label{alg:naive}
	\begin{algorithmic}[1]
		\Require $\big((x_t,u_t)\big)_{t\in[T]}$ (stream), $T$ (stream length), $m$ (max no. of samples per user), $\eps$ (privacy parameter), $\delta$ (failure probability).
		
		\State {\bf Initialize} binary mechanism $\binmech$ with noise level $\eta = \frac{m (1+\log T)}{\eps}$. 
		
		\For{$t = 1,2,\ldots,T$}
		\State $\binmechstream \gets \binmechstream \cup \set{x_t}$
		\State $S_t \gets \binmechsum$
		\State {\bf return} $\est_t = \frac{S_t}{t}$
		\EndFor
	\end{algorithmic}
\end{algorithm}

\paragraph{A better algorithm in a wishful scenario: Exploiting concentration.} 
  Naive use of the binary mechanism for continual mean estimation fails to exploit the concentration phenomenon that results from each user contributing multiple i.i.d. samples to the stream. To see how concentration might help, consider the following scenario. Suppose that (somehow) we already have a prior estimate $\prior$ that satisfies $\abs{\prior-\mean} \leq \frac{1}{\sqrt{m}}$. Also, assume a user ordering where every user contributes their $m$ samples in contiguous time steps. That is, user $1$ contributes samples for the first $m$ time steps, followed by user 2 who contributes samples for the next $m$ time steps, and so on. In this case, Algorithm~\ref{alg:wishful} presents a way to exploit the concentration phenomenon. Even though this algorithm outputs $\est_t$ at every $t$, it only updates $\est_t$ at $t = m, 2m, 3m, \ldots,nm$. Upon receiving a sample from a user, the algorithm does not immediately add it to $\binmechstream$. Instead, the algorithm \textit{waits} for a user to contribute all their $m$ samples. It then computes the sum of those $m$ samples and \textit{projects} the sum to the interval 
\begin{equation} \label{eq:wishfulprojint}
	\cI = 
	\biggl[
	m\prior - \Delta, m\prior + \Delta 
	\biggr] \text{ where } \Delta = \sqrt{\frac{m}{2} \ln \frac{2n}{\delta}} + {\sqrt{m}}.
\end{equation}
before feeding the projected sum to $\binmechstream$. By concentration of sum of i.i.d. Bernoulli random variables,
and the fact that $\abs{\prior-\mean} \leq \frac{1}{\sqrt{m}}$ (by assumption), we have  with probability at least $1-\delta$ that the sum of $m$ samples corresponding to \textit{every} user will fall inside the interval $\cI$; thus, the projection operator $\Pi(\cdot)$ has no effect with high probability.
\medskip

Since there are $n$ users, at most $n$ elements are added to $\binmechstream$ throughout the course of the algorithm. Now, since there can be at most $n$ elements in $\binmechstream$, a given element in $\binmechstream$ will be used at most $1+\log n$ times while computing $\binmechpartial$ (see Section~\ref{subsec:binmech}). So, changing a user can change the $\ell_1$-norm of $\binmechpartial$ by at most $(2\Delta)(1+\log n)$, where $\Delta$ is as in \eqref{eq:wishfulprojint}. Thus, to ensure that the algorithm is $\eps$-DP (given initial estimate $\prior$), it suffices to add independent ${\rm Lap}(\eta)$ noise while computing each term in $\binmechpartial$, where
\begin{equation} \label{eq:wishfulnoise}
	\eta= \frac{2\Delta(1 +\log n)}{\varepsilon}
	=
	\frac{ 2 \paren{\sqrt{\frac{m}{2} \ln \frac{2n}{\delta}} + {\sqrt{m}}}(1+\log n) }{\eps}.
\end{equation}
Note that $\eta$ defined above is proportional to $\sqrt{m}$, whereas the $\eta$ defined in the naive use of binary mechanism was proportional to $m$. Thus (details in Appendix~\ref{supp:wishful}), one obtains a privacy error of $\tilde{O}(\sqrt{m}/t\eps)$, while still having a statistical error of $O(1/\sqrt{t})$ at every $t$.

\begin{algorithm}[ht]
	\caption{Continual mean estimation (wishful scenario)}
	\label{alg:wishful}
	\begin{algorithmic}[1]
		\Require $\big((x_t,u_t)\big)_{t\in[T]}$ (stream \textit{with $m$ samples from any user arriving contiguously}), $n$ (no. of users), $m$ (no. of samples per user), $\eps$ (privacy), $\prior$ (estimate of $\mean$ s.t. $\abs{\prior-\mean} \leq 1/\sqrt{m}$), $\delta$ (failure probability).
		
		\Statex Let $\Pi(\cdot)$ be the projection on the interval $\cI$, where
		\begin{equation*} 
			\cI = 
			\biggl[
			m\prior - \Delta, m\prior + \Delta 
			\biggr] \text{ where } \Delta = \sqrt{\frac{m}{2} \ln \frac{2n}{\delta}} + {\sqrt{m}}.
		\end{equation*}
		
		\State {\bf Initialize} binary mechanism $\binmech$ with noise level $\eta$ where
		\begin{equation*}
			\eta= \frac{2\Delta(1 +\log n)}{\varepsilon}.
		\end{equation*}
		
		\State {\bf Initialize} ${\rm total} \gets 0$.
		\For{$t < m$}
		\State {\bf return} $\est_t = \prior$
		\EndFor
		\For{$t \geq m$}
		\If{$t \in \set{m,2m,3m,\ldots,nm}$}
		\State ${\rm total} \gets {\rm total} + m$
		\State $\sigma = \Pi\paren{\sum_{j=t-m+1}^{t}x_j}$
		\State $\binmechstream \gets \binmechstream \cup \set{\sigma}$
		\EndIf
		\State $S_t \gets \binmechsum$
		\State {\bf return} $\est_t = \frac{S_t}{\rm total}$
		\EndFor
	\end{algorithmic}
\end{algorithm}

The scenario here is wishful for two reasons: (i) we assume a prior estimate $\prior$ that we used to form the interval \eqref{eq:wishfulprojint}; (ii) we assume a special user ordering where every user contributed their samples contiguously. This user ordering ensures that $\est_t$ is updated after every $m$ time steps, and thus, for every $t$, at least $t/2$ samples are used in computing $\est_t$; this is the reason that the statistical error remains $O(1/\sqrt{t})$. Note that this algorithm will perform very poorly in the general user ordering. For instance, if the users contribute samples in a round-robin fashion (where a sample from user $1$ is followed by a sample from user $2$ and so on), then the algorithm would have to wait for time $t=n(m-1)$ to obtain $m$ samples from any given user.
\medskip

Although wishful, there are two main ideas that we take away from this discussion: One is the idea of \textit{withholding}. That is, it is not necessary to incorporate information about every sample received till time $t$ to compute $\est_t$. Second is the idea of \textit{truncation}. That is, the worst-case variation in the quantity of interest due to variation in a user's data can be reduced by projecting user's contributions on a smaller interval. This idea of exploiting concentration using truncation is also used in \cite{Levy_Neurips_2021} for mean estimation in the single-release setting.


\paragraph{Designing algorithms for worst-case user order: Exponential withhold-release pattern.} The algorithm discussed above suffered in the worst-case user ordering since samples from a user were ``withheld'' until {\it every} sample from that user was obtained. This was done to exploit (using ``truncation'') the concentration of sum of $m$ i.i.d. samples from the user. To exploit the concentration of sum of i.i.d. samples in the setting where the user order can be arbitrary, we propose the idea of withholding the samples (and applying truncation) at exponentially increasing intervals. Namely, for a given user $u$, we do not withhold the first two samples $x_1^{(u)}, x_2^{(u)}$; then, we withhold $x_3^{(u)}$ and release truncated version of $(x_3^{(u)} + x_4^{(u)})$ when we receive $x_4^{(u)}$; we then withhold $x_5^{(u)},x_6^{(u)},x_7^{(u)}$ and release truncated version of $(x_5^{(u)}+x_6^{(u)}+x_7^{(u)}+x_8^{(u)})$ when we receive $x_8^{(u)}$; and so on. In general, we withhold samples $x_{2^{\ell-1}+1}^{(u)}, \ldots, x_{2^{\ell}-1}^{(u)}$ and release truncated version of $\sum_{i=2^{\ell-1}+1}^{2^{\ell}}x_i^{(u)}$ when we receive the $2^\ell$-th sample $x_{2^{\ell}}^{(u)}$ from user $u$. 
\medskip
\medskip

We now present Algorithm \ref{alg:singcounter}, which uses this exponential withhold-release idea and, assuming a prior $\prior$, outputs an estimate with statistical error $O(1\sqrt{t})$ and privacy error $\tilde{O}(\sqrt{m}/t\eps)$ for arbitrary user ordering.

\begin{algorithm}[ht]
	\caption{Continual mean estimation assuming prior (single binary mechanism)}
	\label{alg:singcounter}
	\begin{algorithmic}[1]
		\Require $\big((x_t,u_t)\big)_{t\in[T]}$ (stream), $n$ (max no. of users), $m$ (max no. of samples per user), $\eps$ (privacy parameter), $\prior$ (estimate of $\mean$ satisfying $\abs{\prior-\mean} \leq 1/\sqrt{m}$), $\delta$ (failure probability).
		
		\Statex - Let $x^{(u)}_j$ denote the $j$-th sample contributed by user $u$.
		\Statex - For $\ell \geq 1$, let $\Pi_\ell(\cdot)$ be the projection on the interval $\cI_\ell$, where
		\begin{equation*} 
			\cI_\ell:= 
			\brac{ 2^{\ell-1}\prior - \Delta_\ell, 
				2^{\ell-1}\prior + \Delta_\ell }, \quad
			\Delta_\ell = \sqrt{\frac{2^{\ell-1}}{2} \ln \frac{2n\log m}{\delta}} + \frac{2^{\ell-1}}{\sqrt{m}}.
		\end{equation*}
		
		\State {\bf Initialize} binary mechanism $\binmech$ with noise level $\eta(m,n,\delta)$, where
		\begin{equation*} 
			\eta(m,n,\delta) = \frac{ 2 \Delta (1+\log m) \log(1+n(1+\log m)) }{\eps}, \quad
			\Delta = \sqrt{\frac{m}{2} \ln \frac{2n\log m}{\delta}} + \sqrt{m}.
		\end{equation*}
		
		\State For each user $u$, let $M(u) = $ no. of times user $u$ has contributed a sample;
		{\bf initialize} $M(u) \gets 0$.
		
		\State {\bf Initialize} ${\rm total} \gets 0$
		
		\For{$t = 1,2,\ldots,T$}
		\State $M(u_t) \gets M(u_t) + 1$
		\If{$\log M(u_t) \in \Z_+$}
		\State $\ell \gets \log M(u_t)$ \Comment i.e. $M(u_t) = 2^\ell$
		\State ${\rm total} \gets {\rm total} + M(u_t)$				
		\If{$\ell = 0$}
		$\sigma \gets  x_t$
		\ElsIf{$\ell \in \set{1,2,3,\ldots}$}
		\State $\sigma \gets \Pi_{\ell} \paren{\sum_{j=2^{\ell-1}+1}^{2^\ell} x^{(u_t)}_j}$
		\EndIf			
		
		\State $\binmechstream \gets \binmechstream \cup \set{\sigma}$
		\EndIf
		\State $S_t \gets \binmechsum$
		\State {\bf return} $\est_t = \frac{S_t}{\rm total}$
		\EndFor
	\end{algorithmic}
\end{algorithm}

\paragraph{Continual mean estimation assuming prior estimate: Algorithm \ref{alg:singcounter}.} Let $x^{(u)}_j$ be the $j$-th sample obtained from user $u$. Let $\Pi_\ell(\cdot)$ be the projection on the interval $\cI_\ell$, where
\begin{equation} \label{eq:singinterval}
	\cI_\ell:= 
	\brac{ 2^{\ell-1}\prior - \Delta_\ell, 
		2^{\ell-1}\prior + \Delta_\ell }, \quad
	\Delta_\ell = \sqrt{\frac{2^{\ell-1}}{2} \ln \frac{2n\log m}{\delta}} + \frac{2^{\ell-1}}{\sqrt{m}}.
\end{equation}
Let $u_t$ be the user at time $t$. Upon receiving a sample from a user, the algorithm does not immediately add it to $\binmechstream$. Instead, the algorithm only adds anything new to $\binmechstream$ at time instances $t$ when the total number of samples obtained from user $u_t$ becomes $2^\ell$, for $\ell \in \set{0,1,2,\ldots}$. At such a time, the algorithm adds $\Pi_{\ell} \paren{\sum_{j=2^{\ell-1}+1}^{2^\ell} x^{(u_t)}_j}$ to $\binmechstream$. (The summation inside $\Pi_\ell$ makes sense only for $\ell \geq 1$; for $\ell=0$, which corresponds to the first sample from user $u_t$, the algorithm simply adds the sample to $\binmechstream$.) This has the effect that, corresponding to a user, there are at most $(1+\log m)$ elements in $\binmechstream$. Since, there are at most $n$ users, there are at most $n(1+\log m)$ elements added to $\binmechstream$ throughout the course of the algorithm.
\medskip

Now, since there can be at most $n(1+\log m)$ elements in $\binmechstream$, a given element in $\binmechstream$ will be used at most $\log(n(1+\log m))$ times while computing $\binmechpartial$ (see Section~\ref{subsec:binmech}). Thus, changing a user can change the $\ell_1$-norm of $\binmechpartial$ by at most $(1+\log m)(\log(n(1+\log m))) \Delta$, where $\Delta$ is the maximum sensitivity of an element contributed by a user to $\binmechstream$. As can be seen from \eqref{eq:singinterval}, we have $\Delta = 2\paren{\sqrt{\frac{m}{2} \ln \frac{2n\log m}{\delta}} + \sqrt{m}}$. 
Thus, to ensure that Algorithm \ref{alg:singcounter} is $\eps$-DP (given initial estimate $\prior$), it suffices to add independent ${\rm Lap}(\eta(m,n,\delta))$ noise while computing each term in $\binmechpartial$, where
\begin{equation} \label{eq:singcounternoise}
	\eta(m,n,\delta) = \frac{ 2 \Delta (1+\log m) \log(1+n(1+\log m)) }{\eps}, \quad
	\Delta = \sqrt{\frac{m}{2} \ln \frac{2n\log m}{\delta}} + \sqrt{m}.
\end{equation}
Note that, as in \eqref{eq:wishfulnoise}, the magnitude of noise in \eqref{eq:singcounternoise} is $\tilde{O}(\sqrt{m})$. Moreover, with exponential withhold-release pattern, we are guaranteed that at any time $t$, the estimate $\est_t$ is computed using at least $t/2$ samples, no matter what the user ordering is (Claim~\ref{claim:wr} in Appendix~\ref{supp:singcounter}). This gives us the following guarantee (proof in Appendix~\ref{supp:singcounterthm}):
\begin{theorem} \label{thm:singcounter}
	Assume that we are given a user-level $\eps$-DP prior estimate $\prior$ of the true mean $\mean$, such that $\abs{\prior-\mean } \leq \frac{1}{\sqrt{m}}$. Then, Algorithm \ref{alg:singcounter} is $2\eps$-DP. Moreover, for a given $t \in [T]$, we have with probability at least $1-2\delta$,
	\begin{align*}
		\abs{\est_t - \mean} = 
		\tilde{O}\paren{ \frac{1}{\sqrt{t}} + \frac{\sqrt{m}}{t\eps} }.
	\end{align*}
\end{theorem}

We now present Algorithm \ref{alg:multcounter}, which, in conjunction with exponential withhold-release, uses multiple binary mechanisms. Assuming a prior estimate $\prior$, this algorithm outputs an estimate with statistical error $O(1/\sqrt{t})$ and privacy error $\tilde{O}(\sqrt{M_t}/t\eps)$, where $M_t$ is the maximum number of sample contributed by a user. Note that $M_t$ can be much smaller than $m$ for large values of $m$.


\paragraph{Continual mean estimation assuming prior estimate: Algorithm \ref{alg:multcounter}.} In this algorithm, we use $L+1$ binary mechanisms $\binmech[0], \ldots, \binmech[L]$, where $L := \lceil \log m \rceil$. As was the case with Algorithm \ref{alg:singcounter}, here too the algorithm only adds anything new to a binary mechanism at time instances $t$ when the total number of samples obtained from user $u_t$ becomes $2^\ell$, for $\ell \in \set{0,1,2,\ldots}$. What differentiates Algorithm~\ref{alg:multcounter} from Algorithm~\ref{alg:singcounter} is that we use different binary mechanisms for different values of $\ell$. Namely, the element $\Pi_{\ell} \paren{\sum_{j=2^{\ell-1}+1}^{2^\ell} x^{(u)}_j}$ from a user $u$ is added to $\binmech[\ell].{\rm Stream}$.
This ensures that each element in $\binmech[\ell].{\rm Stream}$ has maximum sensitivity $2\Delta_\ell$, where (from \eqref{eq:singinterval}), $\Delta_\ell 
= \sqrt{\frac{2^{\ell-1}}{2} \ln \frac{2n\log m}{\delta}} + \frac{2^{\ell-1}}{\sqrt{m}}$. 
\medskip
\begin{algorithm}[ht]
	\caption{Continual mean estimation assuming prior (multiple binary mechanisms)}
	\label{alg:multcounter}
	\begin{algorithmic}[1]
		\Require $\big((x_t,u_t)\big)_{t\in[T]}$ (stream), $n$ (max no. of users), $m$ (max no. of samples per user), $\eps$ (privacy parameter), $\prior$ (estimate of $\mean$ satisfying $\abs{\prior-\mean} \leq 1/\sqrt{m}$), $\delta$ (failure probability).
		
		\Statex - Let $x^{(u)}_j$ denote the $j$-th sample contributed by user $u$.
		\Statex - For $\ell \in \Z_+$, let $\Pi_\ell(\cdot)$ be the projection on the interval $\cI_\ell$, where
		\begin{equation*} 
			\cI_\ell:= 
			\brac{ 2^{\ell-1}\prior - \Delta_\ell, 
				2^{\ell-1}\prior + \Delta_\ell }, \quad
			\Delta_\ell = \sqrt{\frac{2^{\ell-1}}{2} \ln \frac{2n\log m}{\delta}} + \frac{2^{\ell-1}}{\sqrt{m}}.
		\end{equation*}
		\Statex - Let $L := \lceil \log m \rceil$.
		
		\State {\bf Initialize} $L+1$ binary mechanisms, labelled $\binmech[0], \ldots, \binmech[L]$, where $\binmech[\ell]$ is initialized with noise level $\eta(m,n,\ell,\delta)$, where 
		\begin{align*} 
			\eta(m,n,\ell,\delta) = \frac{ 2\paren{ \sqrt{\frac{2^{\ell-1}}{2} \ln \frac{2n\log m}{\delta}} + \sqrt{2^{\ell-1}} } (1+\log n) }{\eps/(L+1)}.
		\end{align*}
		
		\State For each user $u$, let $M(u) = $ no. of times user $u$ has contributed a sample;
		{\bf initialize} $M(u) \gets 0$.
		
		\State {\bf Initialize} ${\rm total} \gets 0$
		
		\For{$t = 1,2,\ldots,T$}
		\State $M(u_t) \gets M(u_t) + 1$
		\If{$\log M(u_t) \in \Z_+$}
		\State $\ell \gets \log M(u_t)$
		\State ${\rm total} \gets {\rm total} + M(u_t)$
		\If{$\ell = 0$}
		$\sigma \gets  x_t$
		\ElsIf{$\ell \in \set{1,2,3,\ldots}$}
		\State $\sigma \gets \Pi_{\ell} \paren{\sum_{j=2^{\ell-1}+1}^{2^\ell} x^{(u_t)}_j}$
		\EndIf
		
		\State $\binmech[\ell].{\rm Stream}$ $\gets$ $\binmech[\ell].{\rm Stream} \cup \set{\sigma}$
		\EndIf
		\State $S_t \gets \sum_{i=0}^{L}$ $\binmech[i].{\rm Sum}$
		\State {\bf return} $\est_t = \frac{S_t}{\rm total}$
		\EndFor
	\end{algorithmic}
\end{algorithm}

Note that
\begin{align} \label{eq:multsensitivity}
	\Delta_\ell \leq \sqrt{\frac{2^{\ell-1}}{2} \ln \frac{2n\log m}{\delta}} + \sqrt{2^{\ell-1}}
\end{align}
where the inequality holds because $m \geq 2^{\ell-1}$.
To ensure that Algorithm \ref{alg:multcounter} is $\eps$-DP, we will make each binary mechanism $\frac{\eps}{L+1}$-DP. For any $\ell \in \set{0,\ldots,L}$, since each user contributes at most one element to $\binmech[\ell].{\rm Stream}$, there are at most $n$ elements in $\binmech[\ell].{\rm Stream}$ throughout the course of the algorithm. Moreover, since every element in $\binmech[\ell].{\rm Stream}$ has sensitivity at most $2\Delta_\ell$, it suffices to add independent ${\rm Lap}(\eta(m,n,\ell,\delta))$ noise while computing each term in $\binmech[\ell].{\rm NoisyPartialSums}$, where 
\begin{align} \label{eq:multnoise}
	\eta(m,n,\ell,\delta) = \frac{ 2\paren{ \sqrt{\frac{2^{\ell-1}}{2} \ln \frac{2n\log m}{\delta}} + \sqrt{2^{\ell-1}} } (1+\log n) }{\eps/(L+1)}.
\end{align}
Comparing $\eta(m,n,\ell,\delta)$ above to $\eta(m,n,\delta)$ in \eqref{eq:singcounternoise}, we see that using multiple binary mechanisms allows us to add noise with magnitude that is more fine-tuned to exponential withhold-release pattern. In particular, if $M_t$ is the maximum number of samples contributed by any user till time $t$, then at most $\lceil \log M_t \rceil$ binary mechanisms would be in use, and thus the maximum noise level $\eta(m,n,\ell,\delta)$ would be proportional to $\sqrt{M_t}$. This gives us the following guarantee (proof in Appendix~\ref{supp:multcounter}):
\begin{theorem} \label{thm:multcounter}
	Assume that we are given a user-level $\eps$-DP prior estimate $\prior$ of the true mean $\mean$, such that $\abs{\prior-\mean} \leq \frac{1}{\sqrt{m}}$.  Then, Algorithm \ref{alg:multcounter} is $2\eps$-DP. Moreover, for any given $t \in [T]$, we have with probability at least $1-\delta$ that
	\begin{align*}
		\abs{\est_t - \mean} = 
		\tilde{O}\paren{ \frac{1}{\sqrt{t}} + \frac{\sqrt{M_t}}{t\eps} },
	\end{align*}
	where $M_t$ denotes the maximum number of samples obtained from any user till time $t$, i.e., $M_t = \max \set{m_u(t) : u \in [n]}$, where $m_u(t)$ is the number of samples obtained from user $u$ till time $t$.
\end{theorem}

We now present our final algorithm, Algorithm \ref{alg:multcounterfull}, which does not assume a prior estimate. Observe that the prior estimate $\prior$ was needed in previous algorithms only to form truncation intervals $\cI_\ell$ (see \eqref{eq:singinterval}). Algorithm \ref{alg:multcounterfull} includes estimating $\prior$ as a subroutine. In fact, as we describe next, the algorithm estimates a separate prior for each binary mechanism.

\paragraph{Continual mean estimation without assuming prior estimate: Algorithm \ref{alg:multcounterfull}.} Since the algorithm uses $L+1$ (where $L := \lceil \log m \rceil$) binary mechanisms, we need not estimate $\prior$ up to accuracy $O(1/\sqrt{m})$ in one go. Instead, we can have a separate $\prior$ for each binary mechanism, which is helpful since $\binmech[\ell]$ requires $\prior$ only up to an accuracy $O(1/\sqrt{2^{\ell-1}})$ (see \eqref{eq:multsensitivity}). 
\medskip

In this algorithm, we mark $\binmech[\ell]$ as ``inactive'' till we have \textit{sufficient number of users with sufficient number of samples} to estimate a user-level $\frac{\eps}{2L}$-DP prior $\prior_\ell$ up to an accuracy of $\tilde{O}(1/\sqrt{2^{\ell-1}})$. 
While $\binmech[\ell]$ is inactive, we store the elements that require $\prior_\ell$ for truncation in ${\rm Buffer}[\ell]$ (see Line 24-25; each element of ${\rm Buffer}[\ell]$ is a sum of $2^{\ell-1}$ samples, which we cannot truncate yet). Once we have sufficient number of users with sufficient number of samples (Line 9 lists the exact condition), we use a private median estimation algorithm (Algorithm~\ref{alg:pvtmed}) from \cite{Feldman_COLT_2017} to estimate $\prior_\ell$. At this point, we use $\prior_\ell$ to truncate elements stored in ${\rm Buffer}[\ell]$, and pass these elements to $\binmech[\ell].{\rm Stream}$ (Line 11). 

\begin{algorithm}[H]
	\caption{Continual mean estimation: Full algorithm}
	\label{alg:multcounterfull}
	\begin{algorithmic}[1]
		\Require $\big((x_t,u_t)\big)_{t\in[T]}$ (stream), $n$ (max no. of users), $m$ (max no. of samples per user), $\eps$ (privacy parameter), $\delta$ (failure probability)
		
		\Statex Let $x^{(u)}_j$ denote the $j$-th sample contributed by user $u$.
		Let $L := \lceil \log m \rceil$.
		For $\ell \geq 1$, let $\Pi_\ell(\cdot)$ be the projection on the interval $\cI_\ell$ defined as 
		\begin{equation} \label{eq:fullprojinterval}
			\cI_\ell = \brac{ 2^{\ell-1}\prior_\ell -  \Delta_\ell, 2^{\ell-1}\prior_\ell + \Delta_\ell  }
		\end{equation}
		where $\prior_\ell$ is as in Line 10, and
		\begin{equation} \label{eq:fulldelta}
			\Delta_\ell = \sqrt{\frac{2^{\ell-1}}{2} \ln \frac{2n\log m}{\delta/3}} + \sqrt{2^\ell \ln\frac{2k(\frac{\eps}{2L},\ell,\frac{\delta}{3L})}{\delta/3L}}, \quad \text{where }
			k(\eps',\ell,\beta) := \frac{16}{\eps'} \ln\frac{2^{\ell/2}}{\beta}.
		\end{equation}
		
		\State {\bf Initialize} $L+1$ binary mechanisms, labelled $\binmech[0], \ldots, \binmech[L]$, where $\binmech[\ell]$ is initialized with noise level $\eta(m,n,\ell,\delta)$ defined as
		\begin{equation} \label{eq:fullnoise}
			\eta(m,n,\ell,\delta) = 
			\frac{ 2 \Delta_\ell (1+\log n) }{\eps/2(L+1)}
		\end{equation}
		
		\State {\bf Initialize} ${\rm Inactive} \gets \set{2,\ldots,L}$.
		\State {\bf Initialize} $L-1$ buffers labelled ${\rm Buffer}[2],\ldots,{\rm Buffer}[L]$.
		
		\State For each user $u$, let $M(u) = $ no. of times user $u$ has contributed a sample;
		{\bf initialize} $M(u) \gets 0$.
		
		\State {\bf Initialize} ${\rm total} \gets 0$
		
		\For{$t = 1,2,\ldots,T$}
		\State $M(u_t) \gets M(u_t) + 1$
		
		\Statex $\%$ {\it Activating binary mechanisms (if possible):}
		\For{$\ell \in {\rm Inactive}$}
		\If{$\sum_{u=1}^{n} \min \set{M(u), 2^{\ell-1}} \geq$ $2^{\ell-1} \frac{16}{\eps} \paren{2L \ln \frac{3L2^{\ell/2}}{\delta}}$}
		\State $\prior_\ell = {\rm PrivateMedian}\paren{(x_i,u_i)_{i=1}^t, \frac{\eps}{2L},\ell, \frac{\delta}{3L}}$ \Comment{Algorithm~\ref{alg:pvtmed}}
		\State $\binmech[\ell].{\rm Stream} \gets \Pi_\ell\paren{{\rm Buffer}[\ell]}$
		\State ${\rm total} \gets {\rm total} + 2^{\ell-1}{\rm Size}\paren{{\rm Buffer}[\ell]}$
		\EndIf
		\EndFor
		\Statex $\%$
		
		\If{$\log M(u_t) \in \Z_+$}
		\State $\ell \gets \log M(u_t)$
		
		\If{$\ell = 0$}
		$\sigma \gets  x_t$
		\ElsIf{$\ell \in \set{1,2,\ldots,L}$}
		\State $\sigma \gets \sum_{j=2^{\ell-1}+1}^{2^\ell} x^{(u_t)}_j$
		\EndIf
		
		\If{$\ell \notin {\rm Inactive}$}		
		\State $\binmech[\ell].{\rm Stream}$ $\gets$ $\binmech[\ell].{\rm Stream}$ $\cup \set{\Pi_{\ell}\paren{\sigma}}$
		\State ${\rm total} \gets {\rm total} + M(u_t)$
		\ElsIf{$\ell \in {\rm Inactive}$}
		\State ${\rm Buffer}[\ell]$ $\gets$ ${\rm Buffer}[\ell]$ $\cup \set{\sigma}$
		\EndIf
		\EndIf
		\State $S_t \gets \sum_{i=0}^{L}$ $\binmech[i].{\rm Sum}$
		\State {\bf return} $\est_t = \frac{S_t}{\rm total}$
		\EndFor
	\end{algorithmic}
\end{algorithm}
We note that the private median estimation algorithm is also used in \cite{Levy_Neurips_2021} in the single-release setting with all users contributing equal number of samples; we extend this to the setting where number of samples from users can vary. In Appendix~\ref{supp:median} (Claim~\ref{claim:pvtmed}), we show that, for any $\ell \geq 1$, this modified private median estimation algorithm (Algorithm~\ref{alg:pvtmed}) is user-level $\eps$-DP. Moreover, with probability at least $1-\delta-\beta$, we have that
$
\abs{\prior - \mean} \leq 2\sqrt{\frac{1}{2^\ell}\ln\frac{2k(\eps,\ell,\beta)}{\delta}}$, where $k(\eps,\ell,\beta) = \frac{16}{\eps} \ln\frac{2^{\ell/2}}{\beta}$.

\begin{algorithm}[ht]
	\caption{${\rm PrivateMedian}$ (subroutine for Line 10 in Algorithm~\ref{alg:multcounterfull})}
	\label{alg:pvtmed}
	\begin{algorithmic}[1]
		\Require $(x_i,u_i)_{i=1}^t$, $\eps$ (privacy), $\ell$ (scale), $\beta$ (failure probability)
		
		\Ensure $\sum_{u=1}^{n} \min \set{m_u(t), 2^{\ell-1}} \geq2^{\ell-1} \frac{16}{\eps} \ln \frac{2^{\ell/2}}{\beta}$, where $m_u(t)$ is the no. of times user $u$ occurs in $(u_i)_{i=1}^t$.
		
		\Statex
		\State Initialize $k := \paren{\frac{16}{\eps} \ln \frac{2^{\ell/2}}{\beta}}$ arrays $S_1,\ldots,S_k$, each of size $2^{\ell-1}$.
		
		\Statex $\%$ {\it Forming $k$ arrays, each containing $2^{\ell-1}$ samples:}
		\State $j \gets 1$.
		\For{$u \in [n]$}
		\State Let $r = \min \set{m_u(t), 2^{\ell-1}}$.
		\State One by one, start storing $r$ samples from user $u$ in array $S_j$.
		\State At any point, if array $S_j$ becomes full, increment $j$: $j \gets j + 1$.		
		\State Exit loop once array $S_k$ becomes full.
		\Comment\textit{The `Require' condition ensures that this eventually happens.}
		\EndFor 
		\Statex $\%$
		
		\State For $j \in [k]$, let $Y_j$ be the sample mean of all the samples in $S_j$.
		
		\Statex			
		\Statex $\%$ \textit{The steps below are as in \cite{Feldman_COLT_2017},\cite{Levy_Neurips_2021} mutatis mutandis}:
		
		\State Divide the interval $[0,1]$ into disjoint subintervals (``bins''), each of length $(2\cdot 2^{-\ell/2})$. The last subinterval can be of shorter length if $1/(2\cdot 2^{-\ell/2})$ is not an integer. Let $\mathcal{T}$ be the set of middle points of these subintervals.
		
		\State For $j \in [k]$, let $Y'_i = \arg\min_{y \in \mathcal{T}} \abs{Y_i-y}$ be the point in $\mathcal{T}$ closest to $Y_i$.
		
		\State Define cost function $c: \mathcal{T} \to \R$ as
		\[
		c(y) := \max\set{\abs{\set{j \in [k] : Y'_j < y}}, \abs{\set{j \in [k] : Y'_j > y}}}.
		\]
		
		\State Let $\prior$ be a sample drawn from the distribution satisfying
		\[
		\Pr\set{\prior = y} \propto \exp\paren{-\frac{\eps}{4}c(y)}.
		\]
		\Comment\textit{Note that we have $-\frac{\eps}{4}c(y)$ in $\exp(\cdot)$ whereas \cite{Feldman_COLT_2017},\cite{Levy_Neurips_2021} had $-\frac{\eps}{2}c(y)$.}
		\State {\bf return} $\prior$.
	\end{algorithmic}
\end{algorithm}

Algorithm \ref{alg:multcounterfull} demonstrates another advantage of using multiple binary mechanisms -- that we can have different priors for different binary mechanisms, which means that the algorithm does not need to wait for long to start outputting estimates with good guarantees. Theorem \ref{thm:final} (proof in Appendix~\ref{supp:mainthm}) states the  exact guarantees ensured by Algorithm \ref{alg:multcounterfull}. Before stating the theorem, we define what we mean by ``diversity condition''.

{
{
\begin{definition}[Diversity condition] \label{def:diversity}We say that ``diversity condition holds at time $t$'' if 
\begin{equation} \label{eq:diversity}
		\sum_{u=1}^{n}\min\set{m_u(t),\frac{M_t}{2}} \geq \frac{M_t}{2} \frac{16}{\eps} \paren{2L \ln \frac{3L\sqrt{M_t}}{\delta}}
\end{equation}
where $L := \lceil \log m\rceil$ and $M_t$ is the maximum number of samples contributed by any user till time $t$. That is, $M_t := \max \set{m_u(t) : u \in [n]}$, where $m_{u}(t)$ is the number of samples contributed by user $u$ till time $t$.
\end{definition}
In words, condition \eqref{eq:diversity} says that we need the number of users at time $t$ to be large enough so that we can form a collection of $\frac{M_t}{2} \frac{16}{\eps} \paren{2L \ln \frac{3L\sqrt{M_t}}{\delta}}$ samples by using at most $\frac{M_t}{2}$ samples per user. A sufficient condition to ensure \eqref{eq:diversity} holds is that there are at least $\frac{16}{\eps} \paren{2L \ln \frac{3L\sqrt{M_t}}{\delta}}$ users that have contributed at least $\frac{M_t}{2}$ samples each till time $t$. In particular, since $M_t \leq m$, we have the following: Once there are $\frac{16}{\eps} \paren{L \ln \frac{3L\sqrt{m}}{\delta}}$ users that have contributed at least $\frac{m}{2}$ samples each till time $t_0$, then diversity condition holds for every $t \geq t_0$. 

\begin{theorem} \label{thm:final}
	Algorithm \ref{alg:multcounterfull} for continual Bernoulli mean estimation is user-level $\eps$-DP. Moreover, if at time $t \in [T]$, diversity condition \eqref{eq:diversity} holds, then, with probability at least $1-\delta$ (for arbitrary $\delta \in (0,1]$),
	\begin{align*}
		\abs{\est_{t} - \mean} = \tilde{O}\paren{ \frac{1}{\sqrt{t}} + \frac{\sqrt{M_t}}{t\eps} }.
	\end{align*}
\end{theorem}
What happens at time instants when the diversity condition does not hold? This happens when there are very few users who have contributed a very large number of samples. Algorithm~\ref{alg:multcounterfull} stores these samples in buffers (since the corresponding binary mechanisms are ``inactive'') and does not use them to estimate $\est_t$. This is done to preserve user-level privacy and seems like a necessary thing to do. However, currently we do not know whether there is a user-level private way to use these extra samples from few users. In other words, it is not clear if our diversity condition can be weakened.
}}

\section{EXTENSIONS}\label{sec:extensions}

To perform mean estimation on $d$-dimensional inputs with independent coordinates drawn from $\text{Ber}(\mean)$, one can simply use Algorithm \ref{alg:multcounterfull} on each coordinate. Since the release corresponding to each coordinate is user-level $\varepsilon$-DP, we get that the overall algorithm is $d\varepsilon$-DP by basic composition. 
However, if we only require an approximate differential privacy guarantee, then~\cite[Theorem III.3]{Dwork_FOCS_2010} shows that the full sequence of releases from all coordinates will be $\bigg(\varepsilon\sqrt{d\ln(1/\tilde{\delta})},\tilde{\delta}\bigg)$-DP for every $\tilde{\delta}\in(0,1]$.
Rescaling the privacy parameter to ensure $(\varepsilon,\tilde{\delta})$-DP overall gives error $\tilde{O}(1/\sqrt{t}+\sqrt{d M_{t}}/t\varepsilon)$, provided $\sum_{u=1}^{n}\min\{m_{u}(t),M_{t}/2\}\ge\tilde{O}(\sqrt{d M_{t}}/\varepsilon)$.
These arguments carry over to the case of subgaussian distributions as well.

\section{LOWER BOUND}\label{sec:lowerbound}
Consider the single-release mean estimation problem with $n$ users, each having $m$ samples, where the estimated mean must be user-level $\eps$-DP. In this setting, a lower bound of $\Omega\paren{\frac{1}{\sqrt{mn}} + \frac{1}{\sqrt{m}n\eps}}$ on achievable accuracy is known (Theorem 3 in \cite{Liu_neurips_2020}). Furthermore, in the same setting, Theorem 9 in \cite{Levy_Neurips_2021} shows that any algorithm that is user-level $\eps$-DP requires $n = \Omega\paren{\frac{1}{\eps}}$  users. To see what these lower bounds say about our proposed continual-release algorithm (Algorithm \ref{alg:multcounterfull}), let $t$ be a time instant where $N$ users have contributed $M$ samples each (thus, $t = NM$). In this case, $M_t = M$, and Theorem \ref{thm:final} gives an accuracy guarantee of $\tilde{O}\paren{\frac{1}{MN} + \frac{1}{\sqrt{M}N\eps}}$, provided $N = \tilde{\Omega}\paren{\frac{1}{\eps}}$ (diversity condition). This matches the single-release lower bounds (on both accuracy and number of users) up to log factors.
\section{DISCUSSION}\label{sec:discussion}
We have shown that Algorithm \ref{alg:multcounterfull} is almost optimal at every time instant where users have contributed equal number of samples. However, what about the settings when different users contribute different number of samples? Is it optimal, for instance, to not use excessive samples from a single user? The answer is not very clear even in the single-release setting. Investigating this is an interesting future direction.

\printbibliography

\newpage
\appendix

	\section{Useful Inequalities} \label{supp:inequalities}
	
	
	We state two concentration inequalities that we will use extensively.
	\begin{lemma}\label{lem:conc_ber}
		Let $x_{i}\stackrel{iid}{\sim}\text{Ber}(\mean)$ for $i\in[n]$.  Then, for every $\delta\in(0,1]$,
		\begin{align*}
			\mathbb{P}\bigg(\bigg\lvert\sum_{i=1}^{n}x_{i}-n\mean\bigg\rvert \ge \sqrt{\frac{n}{2}\ln\frac{2}{\delta}}\bigg)\le \delta.
		\end{align*}
	\end{lemma}
	
	\begin{lemma}\label{lem:conc_lap}
		Let $x_{i}\sim\text{Lap}(b_{i})$, $i\in[n]$, be independent. Then, for every $\delta\in(0,1]$, 
		\begin{align*}
			\mathbb{P}\bigg(\bigg\lvert\sum_{i=1}^{n}x_{i}\bigg\rvert\ge c\sqrt{\sum_{i=1}^{n}b_{i}^{2}}\ln\frac{1}{\delta}\bigg)\le\delta,
		\end{align*}
		where $c$ is an absolute constant.
	\end{lemma}
	
%
	
%
%
	
	
	\section{Continual mean estimation: Wishful scenario} \label{supp:wishful}
	
	Algorithm~\ref{alg:wishful} is the algorithm under the wishful scenario where:
	\begin{itemize}
		\item we already have a prior estimate $\prior$ that satisfies $\abs{\prior-\mean} \leq \frac{1}{\sqrt{m}}$;
		\item every user contributes their $m$ samples in contiguous time steps; that is, user $1$ contributes samples for the first $m$ time steps, followed by user 2 who contributes samples for the next $m$ time steps, and so on.	
	\end{itemize}
	
%
%
%
	
	\subsection{Guarantees for Algorithm \ref{alg:wishful}}
	Let $x^{(u)}_j$ denote the $j$-th sample contributed by user $u$.
	
	\paragraph{Privacy.} A user $u$ contributes at most one element {to} $\binmechstream$; this element is $\Pi\paren{\sum_{j=1}^{m}x_j^{(u)}}$, where $\Pi(\cdot)$ is the projection on the interval $\cI$ defined in \eqref{eq:wishfulprojint}.
	So, there are at most $n$ elements in $\binmechstream$ throughout the course of the algorithm.
	From the way binary mechanism works, a given element in $\binmechstream$ will be used at most $(1+\log n)$ times while computing terms in $\binmechpartial$ (see Section~2.3 in the main paper). 
	Thus, changing a user can change the $\ell_1$-norm of the array $\binmechpartial$ by at most $(1+\log n) (2\Delta)$, where $\Delta$ is as in \eqref{eq:wishfulprojint}.
	Hence, adding independent ${\rm Lap}(\eta)$ noise (with $\eta$ as in \eqref{eq:wishfulnoise}) while computing each term in $\binmechpartial$ is sufficient to ensure that the array $\binmechpartial$ remains user-level $\eps$-DP throughout the course of the algorithm.
	Since the output $\paren{\est_{t}}_{t=1}^T$ is computed using  $\binmechsum$, which, in turn is a function of the array $\binmechpartial$, we conclude that Algorithm~\ref{alg:wishful} is user-level $\eps$-DP.
	
	
	\paragraph{Utility.} We first show that with high probability, the projection operator $\Pi(\cdot)$ plays no role throughout the algorithm. We then show utility guarantee for $t = km$ assuming no truncation, before generalizing the guarantee to arbitrary $t$.
	\medskip
	
	\underline{\it No truncation happens}:	
	\smallskip
		
	For a user $u$, we have from Lemma~\ref{lem:conc_ber} that
	
	\[
	\Pr\paren{ \abs{\sum_{j=1}^{m}x_j^{(u)} - m\mean} \leq \sqrt{\frac{m}{2} \ln \frac{2n}{\delta}} } \geq 1 - \frac{\delta}{n}.
	\]
	Since $\abs{\prior-\mean} \leq \frac{1}{\sqrt{m}}$, this gives us
	\[
	\forall u \in [n], \ \Pr\paren{ \abs{\sum_{j=1}^{m}x_j^{(u)} - m\prior} \leq \sqrt{\frac{m}{2} \ln \frac{2n}{\delta}} + \sqrt{m}} \geq 1 - \frac{\delta}{n}.
	\]
	Thus, by union bound
	\[
	\Pr\paren{ \forall u \in [n], \abs{\sum_{j=1}^{m}x_j^{(u)} - m\prior} \leq \sqrt{\frac{m}{2} \ln \frac{2n}{\delta}} + \sqrt{m}} \geq 1 - \delta.
	\]
	This means, that with probability at least $1-\delta$, we have $\Pi\paren{\sum_{j=1}^{m}x_j^{(u)}} = \sum_{j=1}^{m}x_j^{(u)}$ for every $u \in [n]$. That is, with probability at least $1-\delta$, no truncation happens throughout the course of the algorithm.
	\medskip
	
	
	\underline{\it Guarantee at $t=km$ ignoring projection $\Pi(\cdot)$}:	
	\smallskip
	
	For now, consider Algorithm~\ref{alg:wishful} without the projection operator $\Pi(\cdot)$ in Line 9. Call it Algorithm~\ref{alg:wishful}-NP (NP $\equiv$ `No Projection').
	For Algorithm~\ref{alg:wishful}-NP, the following happens at $t = km$, for integer $1\leq k\leq n$: $\binmechstream$ has elements
	Let $\sigma_1 = \paren{\sum_{j=1}^{m}x_j^{(1)}}, \sigma_2 =  \paren{\sum_{j=1}^{m}x_j^{(2)}}, \ldots, \sigma_k =  \paren{\sum_{j=1}^{m}x_j^{(k)}}$. Thus, $\binmechpartial$ also contains $k$ terms, where each term has independent $\text{Lap}(\eta)$ added to it. Hence, $\binmechsum$ would be computed using at most $1+\log k$ terms from $\binmechpartial$, and so, using Lemma~\ref{lem:conc_lap}, we have
	\[
	\Pr\paren{\abs{S_{km} - \sum_{i=1}^{k} \sigma_i} \leq c\eta\sqrt{1+\log k} \ln \frac{1}{\delta}} \geq 1-\delta.
	\]
	Furthermore, using Lemma~\ref{lem:conc_ber}, we have for $t = km$ that
	\[
	\Pr\paren{ \abs{\sum_{i=1}^{k} \sigma_i - \mean km} \leq \sqrt{\frac{km}{2} \ln \frac{2}{\delta}} } \geq 1-\delta.
	\]
	Thus, we have the following at $t=km$:
	\[
	\Pr\paren{\abs{S_{km} - \mean km} \leq \sqrt{\frac{km}{2} \ln \frac{2}{\delta}} + c\eta\sqrt{1+\log k} \ln \frac{1}{\delta}} \geq 1-2\delta
	\]
	or, dividing by $km$, we have
	\begin{equation} \label{eq:wish1}
		\Pr\paren{\abs{\est_{km} - \mean} \leq \sqrt{\frac{1}{2km} \ln \frac{2}{\delta}} + c\frac{\eta}{km}\sqrt{1+\log k} \ln \frac{1}{\delta}} \geq 1-2\delta.
	\end{equation}
	
	\underline{\it Guarantee at arbitrary $t\geq m$ ignoring projection $\Pi$}: 	
	\smallskip
	
	We now give utility guarantee of Algorithm~\ref{alg:wishful}-NP, for arbitrary $t \geq m$. Note that for any $t \in [km, (k+1)m-1)$, $k \geq 1$, we output $\est_{t} = \est_{km}$. Moreover, we always have that $km \geq \frac{t}{2}$. Thus, for $t \in [km, (k+1)m-1)$, we have, with probability $\geq 1-2\delta$
	
	\begin{align} \label{eq:wish2}
		\abs{\est_{t} - \mean} 
		&= \abs{\est_{km} - \mean} \nonumber\\
		&\leq \sqrt{\frac{1}{2km} \ln \frac{2}{\delta}} + c\frac{\eta}{km}\sqrt{1+\log k} \ln \frac{1}{\delta} \tag{from \eqref{eq:wish1}} \nonumber\\
		&\leq \sqrt{\frac{1}{t} \ln \frac{2}{\delta}} + c\frac{2\eta}{t}\sqrt{1+\log k} \ln \frac{1}{\delta}. \tag{since $km \geq \frac{t}{2}$}
	\end{align}

	\underline{\it Final utility guarantee for Algorithm~\ref{alg:wishful} at arbitrary $t$}: 	
	\smallskip
	
	The above utility guarantee was obtained for Algorithm~\ref{alg:wishful}-NP, which is a variant of Algorithm~\ref{alg:wishful} without projection operator $\Pi(\cdot)$ in Line 9. Now, let $\cE_1$ be the event that no truncation happens. We already saw that no truncation happens (i.e. projection operator plays no role) with probability at least $1-\delta$. That is, $\Pr(\cE_1) \geq 1-\delta$. Observe that
	\begin{align*}
	\cE_1 
	&:= \set{\text{no truncation happens}} \\
	&= \set{\text{Algorithm~\ref{alg:wishful} and Algorithm~\ref{alg:wishful}-NP become equivalent}}.
	\end{align*}
	Let
	\[
	\cE_2 
	= \set{\abs{\est_{t} - \mean} \leq \sqrt{\frac{1}{t} \ln \frac{2}{\delta}} + c\frac{2\eta}{t}\sqrt{1+\log k} \ln \frac{1}{\delta} \text{ for Algorithm~\ref{alg:wishful}-NP}}
	\]
	We saw above that $\Pr(\cE_2) \geq 1-2\delta$. Thus, using union bound that $\Pr(\cE_1 \cap \cE_2) \geq 1-3\delta$. That is, for Algorithm~\ref{alg:wishful}, we have that for $t \geq m$, (we upper bound $\log k$ by $\log n$)
	\[
	\Pr\paren{ \abs{\est_{t} - \mean} \leq \sqrt{\frac{1}{t} \ln \frac{2}{\delta}} + c\frac{2\eta}{t}\sqrt{1+\log n} \ln \frac{1}{\delta} } \geq 1-3\delta.
	\]
	Substituting value of $\eta$ from \eqref{eq:wishfulnoise}, we get that for $t \geq m$, with probability at least $1-3\delta$, we have
	\begin{align*}
		\abs{\est_{t} - \mean} 
		&\leq \sqrt{\frac{1}{t} \ln \frac{2}{\delta}} + \frac{\sqrt{m}}{t\eps}\paren{ 1+\sqrt{\frac{1}{2} \ln \frac{2n}{\delta}} } {4c(1+\log n)^{3/2} }\ln \frac{1}{\delta} \\
		&= \tilde{O}\paren{\frac{1}{\sqrt{t}} + \frac{\sqrt{m}}{t\eps}}.
	\end{align*}
	This guarantee holds trivially for $t < m$ as well because the algorithm outputs the prior estimate $\prior$ (Lines 3-5), which gives us that $\abs{\est_t - \mean} = \abs{\prior-\mean} \leq \frac{1}{\sqrt{m}} \leq \frac{1}{\sqrt{t}}$.

	
	\medskip
	{\it Remark:} Throughout this section, we assumed that we were given a prior estimate $\prior$ satisfying $\abs{\prior-\mean} \leq \frac{1}{\sqrt{m}}$. Our discussion here also applies to the case even when we have a prior that satisfies $\abs{\prior-\mean} \leq \frac{1}{\sqrt{m}}$ with probability at least $1-\delta$ (instead of being deterministic). Also, if $\prior$ is computed using samples from users, it  should be user-level DP. In particular, if $\prior$ is user-level $\eps$-DP, the overall algorithm becomes user-level $2\eps$-DP (using composition property of DP from Lemma~\ref{lem:composition}).
	
	\newpage
	\section{Continual mean estimation assuming prior estimate: Single binary mechanism} \label{supp:singcounter}

	Before proving guarantees for Algorithm~\ref{alg:singcounter} (Theorem~3.1), we prove a claim about exponential withhold-release pattern with arbitrary user ordering.
	
	\subsection{Exponential withhold-release}
	Recall the exponential withhold-release pattern. For a given user $u$, we release the first two samples $x_1^{(u)}, x_2^{(u)}$; then, we withhold $x_3^{(u)}$ and release truncated version of $(x_3^{(u)} + x_4^{(u)})$ when we receive $x_4^{(u)}$; we then withhold $x_5^{(u)},x_6^{(u)},x_7^{(u)}$ and release truncated version of $(x_5^{(u)}+x_6^{(u)}+x_7^{(u)}+x_8^{(u)})$ when we receive $x_8^{(u)}$; and so on. In general, we withhold samples $x_{2^{\ell-1}+1}^{(u)}, \ldots, x_{2^{\ell}-1}^{(u)}$ and release truncated version of $\paren{\sum_{i=2^{\ell-1}+1}^{2^{\ell}}x_i^{(u)}}$ when we receive the $2^\ell$-th sample $x_{2^{\ell}}^{(u)}$ from user $u$.
	
	We ignore truncations for now. For a user $u$, let $\sigma_0^{(u)} = x_1^{(u)},$ $\sigma_1^{(u)} = x_2^{(u)},$ $\sigma_2^{(u)} = (x_3^{(u)} + x_4^{(u)}),\ldots,$ $\sigma_{\ell}^{(u)}~=~ \paren{\sum_{i=2^{\ell-1}+1}^{2^{\ell}}x_i^{(u)}}, \ldots$. Let $\big((x_t,u_t)\big)_{t\in[T]}$ be a stream with arbitrary user order. We follow the exponential withhold-release protocol and feed $\sigma_\ell^{(u)}$'s to an array named ${\rm Stream}$. For instance, suppose we receive $(x_1,1), (x_2,2), (x_3,2), (x_4,2), (x_5,1),(x_6,2), (x_7,1), (x_8,1)$, where the second index denotes the user identity. Then, the input feed to ${\rm Stream}$ looks as follows:
	\begin{equation} \label{eq:wrexample}
		\begin{matrix}
			t=1: & x_1 & (\text{1st sample from user }1)\\
			t=2: & x_2 & (\text{1st sample from user }2)\\
			t=3: & x_3 & (\text{2nd sample from user }2)\\
			t=4: &{\tt withhold} & (\text{3rd sample from user }2)\\
			t=5: & x_5 & (\text{2nd sample from user }1)\\
			t=6: & x_4+x_6 & (\text{4th sample from user }2)\\
			t=7: &{\tt withhold} & (\text{3rd sample from user }1)\\
			t=8: &x_7+x_8 & (\text{4th sample from user }1)
		\end{matrix}
	\end{equation}
	Since we are only interested in computing sums, feeding $\sigma_{\ell}^{(u)} = \paren{\sum_{i=2^{\ell-1}+1}^{2^{\ell}}x_i^{(u)}}$ to ${\rm Stream}$ is equivalent to feeding {\it information about $2^{\ell}-1$ samples} to ${\rm Stream}$. We now claim the following.
	
	\begin{claim} \label{claim:wr}
		Let $\big((x_t,u_t)\big)_{t\in[T]}$ have arbitrary user ordering. Suppose we follow the exponential withhold-release protocol and feed $\sigma_\ell^{(u)}$'s to an array named ${\rm Stream}$. Then, at any time $t$, ${\rm Stream}$ contains information about at least $t/2$ samples.
	\end{claim}
	\begin{proof}
		Let `R' denote `release' and `W' denote `withhold'. Suppose, for now, only user $u$ arrives at all time steps. Then, the withhold-release sequence looks like $(R_u,R_u,W_u,R_u,3W_u,R_u,7W_u,R_u,\cdots)$. Here, `$kW_u$' denotes `$W_u$' occurs for the next $k$ steps in the sequence. Moreover, at every `$R_u$', information about the samples withheld after previous `$R_u$' are released. Note that, at any point along this withhold-release sequence, the number of samples withheld is less than or equal to the number of samples whose information has been released. This is for a given user $u$.
		\medskip
		
		Now, consider a withhold-release sequence induced by an arbitrary user ordering. E.g., if the user order is $1,2,2,2,1,2,1,1$, then the corresponding withhold-release sequence is $(R_1,R_2,R_2,W_2,R_1,R_2,W_1,R_1)$ (see \eqref{eq:wrexample}). Here, subscript denotes user ID. Now, at any time $t$ in this withhold-release sequence, if we just consider $R_u$'s and $W_u$'s up to time $t$ for a fixed user $u$ , we will have that (argued above) the number of samples withheld for user $u$ is less than or equal to the number of samples from user $u$ whose information has been released. This holds for every user who has occurred till time $t$. Thus, at any time $t$, total number of samples withheld {\it (across all users)} is less than or equal to total number of samples {\it (across all users)} whose information has been released. Hence, ${\rm Stream}$ contains information about at least $t/2$ samples.
	\end{proof}
	
	\subsection{Proof of Theorem~3.1} \label{subsec:3.1} \label{supp:singcounterthm}
	We will prove the following theorem.
	
	\begin{theorem*}
		Assume that we are given a user-level $\eps$-DP prior estimate $\prior$ of the true mean $\mean$, such that $\abs{\prior-\mean } \leq \frac{1}{\sqrt{m}}$. Then, Algorithm \ref{alg:singcounter} is $2\eps$-DP. Moreover, for any given $t \in [T]$, we have with probability at least $1-\delta$ that
		\begin{align*}
			\abs{\est_t - \mean} 
			= \tilde{O}\paren{ \frac{1}{\sqrt{t}} + \frac{\sqrt{m}}{t\eps} }.
		\end{align*}
	\end{theorem*}
	\begin{proof}
		We fist prove the privacy guarantee and then prove the utility guarantee.
		
		\paragraph{Privacy.} The prior estimate $\prior$ is given to be user-level $\eps$-DP. We will now show that the array \\$\binmechpartial$ is user-level $\eps$-DP throughout the course of the algorithm. Note that the output estimates $\est_t$ are computed using $\binmechsum$, which is a function of $\binmechpartial$. Thus, if $\prior$ and $\binmechpartial$ are both user-level $\eps$-DP, by composition property (Lemma~\ref{lem:composition}), the overall output $\paren{\est_t}_{t=1}^T$ will be user-level $2\eps$-DP.
		\medskip
		
		\underline{\it Proof that $\binmechpartial$ is user-level $\eps$-DP:}  	
		\smallskip
		
		Since a user $u$ contributes at most $m$ samples, and does so in an exponential withhold-release pattern, there are at most $(1+\log m)$ elements in $\binmechstream$ corresponding to user $u$. Since there are at most $n$ users, there are at most $n(1+\log m)$ elements added to $\binmechstream$ throughout the course of the algorithm.
		\medskip
		
		Now, since there can be at most $n(1+\log m)$ elements in $\binmechstream$, a given element in $\binmechstream$ will be used at most $\log(1+n(1+\log m))$ times while computing $\binmechpartial$. Thus, changing a user can change the $\ell_1$-norm of $\binmechpartial$ by at most $(1+\log m)(\log(1+n(1+\log m))) \Delta$, where $\Delta$ is the maximum sensitivity of an element contributed by a user to $\binmechstream$. As can be seen from \eqref{eq:singinterval}, we have $\Delta_\ell \leq \paren{\sqrt{\frac{m}{2} \ln \frac{2n\log m}{\delta}} + \sqrt{m}}$ for every $\ell$. Thus, $\Delta =  2\paren{\sqrt{\frac{m}{2} \ln \frac{2n\log m}{\delta}} + \sqrt{m}}$ is an upper bound on worst-case sensitivity of an element contributed by a user to $\binmechstream$.
		Hence, adding independent ${\rm Lap}(\eta(m,n,\delta))$ noise (with $\eta(m,n,\delta)$ as in \eqref{eq:singcounternoise}) while computing each term in $\binmechpartial$ is sufficient to ensure that $\binmechpartial$ is user-level $\eps$-DP.
		
		\paragraph{Utility.} We first show that with high probability, the projection operators $\Pi_\ell(\cdot)$ play no role throughout the course of the algorithm.
		\medskip
		
		\underline{\it No truncation happens}:  	
		\smallskip
		
		For a user $u$, for any $\ell \geq 1$, we have from Lemma~\ref{lem:conc_ber} that
		\[
		\Pr\paren{ \abs{\paren{\sum_{j=2^{\ell-1}+1}^{2^\ell}x_j^{(u)}} - (2^{\ell-1})\mean} \leq \sqrt{\frac{2^{\ell-1}}{2} \ln \frac{2n \log m}{\delta}} } \geq 1 - \frac{\delta}{n\log m}.
		\]
		Since $\abs{\prior-\mean} \leq \frac{1}{\sqrt{m}}$, this gives us that, for any user $u$, for any $\ell \geq 1$,
		\[
		\Pr\paren{ \abs{\paren{\sum_{j=2^{\ell-1}+1}^{2^\ell}x_j^{(u)}} - (2^{\ell-1})\prior} \leq \sqrt{\frac{2^{\ell-1}}{2} \ln \frac{2n \log m}{\delta}} + \frac{2^{\ell-1}}{\sqrt{m}}} \geq 1 - \frac{\delta}{n}.
		\]
		Note that a user contributes at most $m$ samples. Thus, at most $\log m$ projection operators $\Pi_\ell(\cdot)$ are applied per user. Applying union bound (over $\ell$), we have that, for any user $u$
		\[
		\Pr\paren{ \forall \ell \in \set{1,\ldots,\lfloor \log m \rfloor}, \abs{\paren{\sum_{j=2^{\ell-1}+1}^{2^\ell}x_j^{(u)}} - (2^{\ell-1})\prior} \leq \sqrt{\frac{2^{\ell-1}}{2} \ln \frac{2n \log m}{\delta}} + \frac{2^{\ell-1}}{\sqrt{m}} } \geq 1 - \frac{\delta}{n}.
		\]
		Now, we take union bound over $n$ users, which gives us that
		\[
		\Pr\paren{ \forall u \in [n], \forall \ell \in \set{1,\ldots,\lfloor \log m \rfloor}, \abs{\paren{\sum_{j=2^{\ell-1}+1}^{2^\ell}x_j^{(u)}} - (2^{\ell-1})\prior} \leq \sqrt{\frac{2^{\ell-1}}{2} \ln \frac{2n \log m}{\delta}} + \frac{2^{\ell-1}}{\sqrt{m}}} \geq 1 - \delta.
		\]
		Note that the projection operator $\Pi_{\ell}(\cdot)$ was defined as projection on interval $\cI_\ell$ as in \eqref{eq:singinterval}. The above equation shows that with probability at least $1-\delta$, the projection operators do not play any role throughout the algorithm, and thus no truncation happens.
		\medskip
		
		\underline{\it Utility at time $t$ ignoring projections $\Pi_{\ell}$}:  	
		\smallskip
		
		For now, consider Algorithm~\ref{alg:singcounter} without the projection operator $\Pi_{\ell}(\cdot)$ in Line 11. Call it Algorithm~\ref{alg:singcounter}-NP (NP $\equiv$ `No Projection'). For Algorithm~\ref{alg:singcounter}-NP, at any time $t$, $\binmechstream$ has terms of the form $\sigma_{\ell}^{(u)} = \paren{\sum_{i=2^{\ell-1}+1}^{2^{\ell}}x_i^{(u)}}$ (note that $\sigma_{\ell}^{(u)}$ ``contains information'' about $2^{\ell-1}$ samples from user $u$). From Claim~\ref{claim:wr}, we have that, at time $t$, $\binmechstream$ contains information about at least $t/2$ samples. Thus, $S_t$ would be sum of at least $t/2$ samples with Laplace noises added to it. 
		\medskip
		
		Note that, at any time $t$, at most $t$ elements are present in $\binmechstream$. This also means that at most $t$ elements are present in $\binmechpartial$. Thus, computing $S_t$ (using $\binmechsum$) would involve at most $(1+\log t)$ terms from $\binmechpartial$. Each term in $\binmechpartial$ has independent ${\rm Lap}(\eta(m,n,\delta))$ noise added to it, where $\eta(m,n,\delta)$ is as in \eqref{eq:singcounternoise}. 
		\medskip
		
		Thus, at time $t$, for Algorithm~\ref{alg:singcounter}-NP,
		\[
		S_t = \brac{\text{sum of at least $t/2$ Bernoulli samples}} + \brac{\text{sum of at most $(1+\log t)$ i.i.d. ${\rm Lap}(\eta(m,n,\delta))$ terms}}
		\]
		Hence, using Lemma~\ref{lem:conc_ber} and Lemma~\ref{lem:conc_lap}, we get that at time $t$,
		\[
		\Pr\paren{ \abs{\est_t - \mean} \leq \sqrt{\frac{1}{t} \ln\frac{2}{\delta}} + c\eta(m,n,\delta)\sqrt{1+\log t}\ln\frac{1}{\delta} } \geq 1 - 2\delta
		\]
		
		\underline{\it Final utility guarantee for Algorithm~\ref{alg:singcounter} at time $t$}:  	
		\smallskip
		
		The above utility guarantee was obtained for Algorithm~\ref{alg:singcounter}-NP, which is a variant of Algorithm~\ref{alg:singcounter} without projection operator $\Pi_\ell(\cdot)$ in Line 11. Now, let $\cE_1$ be the event that no truncation happens. We already saw that no truncation happens (i.e. projection operator plays no role) with probability at least $1-\delta$. That is, $\Pr(\cE_1) \geq 1-\delta$. Observe that
		\begin{align*}
			\cE_1 
			&:= \set{\text{no truncation happens}} \\
			&= \set{\text{Algorithm~\ref{alg:singcounter} and Algorithm~\ref{alg:singcounter}-NP become equivalent}}.
		\end{align*}
		Let
		\[
		\cE_2 
		= \set{\abs{\est_t - \mean} \leq \sqrt{\frac{1}{t} \ln\frac{2}{\delta}} + c\eta(m,n,\delta)\sqrt{1+\log t}\ln\frac{1}{\delta} \text{ for Algorithm~\ref{alg:singcounter}-NP}}
		\]
		We saw above that $\Pr(\cE_2) \geq 1-2\delta$. Thus, using union bound that $\Pr(\cE_1 \cap \cE_2) \geq 1-3\delta$. That is, for Algorithm~\ref{alg:singcounter}, we have, with probability at least $1-3\delta$,
		\begin{align*}
			\abs{\est_t - \mean} 
			&\leq \sqrt{\frac{1}{t} \ln\frac{2}{\delta}} + c\frac{\eta(m,n,\delta)}{t}\sqrt{1+\log t}\ln\frac{1}{\delta} \\
			&= O\paren{ \sqrt{\frac{1}{t} \ln\frac{1}{\delta}} + \frac{\eta(m,n,\delta)}{t}\sqrt{\log t} \ln\frac{1}{\delta} } \\
			&= O\paren{ \sqrt{\frac{1}{t} \ln\frac{1}{\delta}} + \frac{\sqrt{m}}{t\eps} \sqrt{\log t} \log m \log(n\log m) \sqrt{\ln\frac{n\log m}{\delta}} \ln\frac{1}{\delta} } \tag{using $\eta(m,n,\delta)$ from \eqref{eq:singcounternoise}} \\
			&= \tilde{O}\paren{ \frac{1}{\sqrt{t}} + \frac{\sqrt{m}}{t\eps} }.
		\end{align*}
	\end{proof}
	
	
	\newpage
	\section{Continual mean estimation assuming prior: Multiple binary mechanisms} \label{supp:multcounter}

	\subsection{Proof of Theorem~3.2} \label{subsec:3.2}
	We will prove the following theorem.
	\begin{theorem*}
		Assume that we are given a user-level $\eps$-DP prior estimate $\prior$ of the true mean $\mean$, such that $\abs{\prior-\mean} \leq \frac{1}{\sqrt{m}}$.  Then, Algorithm \ref{alg:multcounter} is $2\eps$-DP. Moreover, for any given $t \in [T]$, we have with probability at least $1-\delta$ that
		\begin{align*}
			\abs{\est_t - \mean} = 
			\tilde{O}\paren{ \frac{1}{\sqrt{t}} + \frac{\sqrt{M_t}}{t\eps} },
		\end{align*}
		where $M_t$ denotes the maximum number of samples obtained from any user till time $t$, i.e., $M_t = \max \set{m_u(t) : u \in [n]}$, where $m_u(t)$ is the number of samples obtained from user $u$ till time $t$.
	\end{theorem*}
	\begin{proof}
		We fist prove the privacy guarantee and then prove the utility guarantee. Let $L := \lfloor \log m \rfloor$.
		
		\paragraph{Privacy.} The prior estimate $\prior$ is given to be user-level $\eps$-DP. We will now show that, for each $\ell \in \set{0,\ldots,L}$, the array $\binmech[\ell].{\rm NoisyPartialSums}$ is user-level $\frac{\eps}{L+1}$-DP. By composition property (Lemma~\ref{lem:composition}), this would mean that $\Big( \binmech[\ell].{\rm NoisyPartialSums} \Big)_{\ell=0}^L$ is user-level $\eps$-DP. Note that the output estimates $\est_t$ are computed using $\Big( \binmech[\ell].{\rm Sum} \Big)_{\ell=0}^L$, which is a function of \\$\Big( \binmech[\ell].{\rm NoisyPartialSums} \Big)_{\ell=0}^L$. Thus, if $\Big( \binmech[\ell].{\rm NoisyPartialSums} \Big)_{\ell=0}^L$ is user-level $\eps$-DP, and prior estimate $\prior$ is also user-level $\eps$-DP, the overall output $\paren{\est_t}_{t=1}^T$ is guaranteed to be user-level $2\eps$-DP.
		\medskip
		
		\underline{\it Proof that $\binmech[\ell].{\rm NoisyPartialSums}$ is user-level $\frac{\eps}{L+1}$-DP}: 	
		\smallskip
		
		Consider $\binmech[\ell]$. A user $u$ contributes at most one element to $\binmech[\ell].{\rm Stream}$; this element is $\Pi_\ell\paren{\sum_{j=2^{\ell-1}+1}^{2^\ell} x^{(u)}_j}$, where $\Pi_\ell(\cdot)$ is the projection on the interval $\cI_\ell$ defined in \ref{eq:singinterval}. 
		So, there are at most $n$ elements in $\binmech[\ell].{\rm Stream}$ throughout the course of the algorithm. 
		A given element in $\binmech[\ell].{\rm Stream}$ will be used at most $(1+\log n)$ times while computing terms in $\binmech[\ell].{\rm NoisyPartialSums}$. Thus, changing a user can change the $\ell_1$-norm of the array \\$\binmech[\ell].{\rm NoisyPartialSums}$ by at most $(1+\log n) (2\Delta_\ell)$, where $\Delta_\ell$ is as in \eqref{eq:singinterval}. 
		Hence, adding independent ${\rm Lap}(\eta(m,n,\ell,\delta))$ noise (with $\eta(m,n,\ell,\delta)$ as in \eqref{eq:multnoise}) while computing each term in \\$\binmech[\ell].{\rm NoisyPartialSums}$ is sufficient to ensure that the array $\binmech[\ell].{\rm NoisyPartialSums}$ remains user-level $\frac{\eps}{L+1}$-DP throughout the course of the algorithm.	
		
		\paragraph{Utility.} Exactly as in the utility proof of Theorem~3.1 (Section~\ref{subsec:3.1}), we have that with probability at least $1-\delta$, the projection operators do not play any role throughout the algorithm, and thus no truncation happens.
		\medskip
		
		\underline{\it Utility at time $t$ ignoring projections $\Pi_\ell$}:  	
		\smallskip
		
		For now, consider Algorithm~\ref{alg:multcounter} without the projection operator $\Pi_{\ell}(\cdot)$ in Line 11. Call it Algorithm~\ref{alg:multcounter}-NP (NP $\equiv$ `No Projection'). We will derive utility at time $t$ for Algorithm~\ref{alg:multcounter}-NP.
		\medskip
		
		The only difference between Algorithm~\ref{alg:multcounter}-NP and Algorithm~\ref{alg:singcounter}-NP is that the term $\paren{\sum_{j=2^{\ell-1}+1}^{2^\ell} x^{(u)}_j}$ from user $u$ is fed to $\binmech[\ell].{\rm Stream}$ instead of feeding it to a common binary mechanism stream. So, for Algorithm~\ref{alg:multcounter}-NP, using Claim~\ref{claim:wr}, we have that, at time $t$, the combined streams $\big(\binmech[\ell].{\rm Stream}\big)_{\ell=0}^L$ contain information about at least $t/2$ samples. Thus, $S_t$ would be sum of at least $t/2$ samples with Laplace noises added to it. 
		\medskip
		
		Now, since every user has contributed at most $M_t$ samples, it follows that $\binmech[\ell].{\rm Stream}$ for $\ell > \lfloor \log M_t \rfloor$ will be empty. That is, only $\binmech[0],\ldots,\binmech[\lfloor \log M_t \rfloor]$ will contribute to the sum $S_t$ at time $t$. Moreover, for $\ell \leq \lfloor \log M_t \rfloor$, since each user contributes at most one item to $\binmech[\ell].{\rm Stream}$, there can be at most $n$ terms in $\binmech[\ell].{\rm Stream}$ and in $\binmech[\ell].{\rm NoisyPartialSums}$. Thus, computing $\binmech[\ell].{\rm Sum}$ at time $t$ involves at most $(1+\log n)$ terms from $\binmech[\ell].{\rm NoisyPartialSums}$. Each term in $\binmech[\ell].{\rm NoisyPartialSums}$ has independent ${\rm Lap}(\eta(m,n,\ell,\delta))$ noise added to it, where $\eta(m,n,\ell,\delta)$ is as in \eqref{eq:multnoise}.
		\medskip
		
		
		Thus, at time $t$, for Algorithm~\ref{alg:multcounter}-NP,
		\[
		S_t = \brac{\text{sum of at least $t/2$ Bernoulli samples}} + \brac{\sum_{\ell=0}^{\lfloor \log M_t \rfloor}\text{sum of at most $(1+\log n)$ i.i.d. ${\rm Lap}(\eta(m,n,\ell,\delta))$ terms}}
		\]
		Hence, using Lemma~\ref{lem:conc_ber} and Lemma~\ref{lem:conc_lap}, we get that at time $t$,
		\[
		\Pr\paren{ \abs{\est_t - \mean} \leq \sqrt{\frac{1}{t} \ln\frac{2}{\delta}} + \frac{c}{t} \ln\frac{1}{\delta} \sqrt{\sum_{\ell=0}^{\lfloor \log M_t \rfloor}(1+\log n)\eta(m,n,\ell,\delta)^2} } \geq 1 - 2\delta
		\]
		
		\underline{\it Final utility guarantee for Algorithm~\ref{alg:multcounter} at time $t$}:  	
		\smallskip
		
		The above utility guarantee was obtained for Algorithm~\ref{alg:multcounter}-NP, which is a variant of Algorithm~\ref{alg:multcounter} without projection operator $\Pi_\ell(\cdot)$ in Line 11. But, indeed, with probability at least $1-\delta$, no truncation happens. Proceeding as in the proof of Theorem~3.1 (Section~\ref{subsec:3.1}), we take a union bound, and get that, with probability at least $1-3\delta$, Algorithm~\ref{alg:multcounter} satisfies
		\begin{align*}
			\abs{\est_t - \mean} 
			&\leq \sqrt{\frac{1}{t} \ln\frac{2}{\delta}} + \frac{c}{t} \ln\frac{1}{\delta} \sqrt{\sum_{\ell=0}^{\lfloor \log M_t \rfloor}(1+\log n)\eta(m,n,\ell,\delta)^2} \\
			&= O\paren{ \sqrt{\frac{1}{t} \ln\frac{1}{\delta}} + \frac{\eta(m,n,\lfloor \log M_t \rfloor,\delta)}{t}\sqrt{\log n} \ln\frac{1}{\delta} } \\
			&= O\paren{ \sqrt{\frac{1}{t} \ln\frac{1}{\delta}} + \frac{\sqrt{M_t}}{t\eps} (\log m) (\log n)^{3/2} \sqrt{\ln\frac{n\log m}{\delta}} \ln\frac{1}{\delta} } \tag{from \eqref{eq:multnoise}} \\
			&= \tilde{O}\paren{ \frac{1}{\sqrt{t}} + \frac{\sqrt{M_t}}{t\eps} }.
		\end{align*}
	\end{proof}
	
	\newpage
	
	\section{Continual mean estimation: Full algorithm} \label{supp:fullcounter}

	\renewcommand{\algorithmicensure}{\textbf{Require:}}
	\subsection{Private Median Algorithm} \label{supp:median}

	\begin{claim} \label{claim:pvtmed}
		For any $\ell \geq 1$, Algorithm~\ref{alg:pvtmed} (${\rm PrivateMedian}$) is user-level $\eps$-DP. Moreover, with probability at least $1-\delta-\beta$, 
		\[
		\abs{\prior - \mean} \leq 2\sqrt{\frac{1}{2^\ell}\ln\frac{2k(\eps,\ell,\beta)}{\delta}}
		\]
		where $k(\eps,\ell,\beta) = \frac{16}{\eps} \ln\frac{2^{\ell/2}}{\beta}$.
	\end{claim}
	\begin{proof}
		Let $k := k(\eps,\ell,\beta)$.
		
		\paragraph{Privacy.} From the way arrays $S_1,\ldots,S_k$ in Algorithm~\ref{alg:pvtmed} are created (Lines 3-8), it follows that samples from any given user $u$ appears in at most $2$ arrays. This is because: (i) each array contains $2^{\ell-1}$ samples; and (ii) each user contributes at most $2^{\ell-1}$ samples (see definition of $r$ in Line 4); (iii) samples from a user are added contiguously to arrays (see Lines 5-6). Now, for $j \in [k]$, since $Y_j$ is the average of samples in array $S_j$, and $Y'_j$ is a quantized version of $Y_j$, it follows that changing a user changes at most $2$ elements out of $\set{Y'_1,\ldots,Y'_k}$. Thus, for any $y \in \mathcal{T}$, the cost $c(y)$ can vary by at most $2$ if a user is changed. Since the worst-case sensitivity (w.r.t. change of a user) of cost $c$ is $\Delta c:=2$, exponential mechanism with sampling probability proportional to $\exp\paren{-\frac{\eps}{2\Delta c}c(y)}$ is $\eps$-DP (\cite{Dwork_Now_2014}) w.r.t. change of a user. This proves that Algorithm~\ref{alg:pvtmed} is user-level $\eps$-DP.
		
		\paragraph{Utility.} Since for each $j \in [k]$, $Y_j$ is sample mean of $2^{\ell-1}$ Bernoulli random variables, we have by Lemma~\ref{lem:conc_ber} that
		\[
		\forall j \in [k], \ \Pr\paren{ \abs{Y_j-\mean} \leq \sqrt{\frac{1}{2^\ell}\ln \frac{2k}{\delta}} } \geq 1-\frac{\delta}{k}.
		\]
		Thus, by union bound,
		\[
		\Pr\paren{ \forall j \in [k], \ \abs{Y_j-\mean} \leq \sqrt{\frac{1}{2^\ell}\ln \frac{2k}{\delta}} } \geq 1-\delta.
		\]
		Since $\abs{Y'_j-\mean} \leq \abs{Y'_j-Y_j} + \abs{Y_j-\mean}$, and $\abs{Y'_j-Y_j} \leq 2^{-\ell/2}$, it follows that
		\begin{equation} \label{eq:pvtmed_quantize}
			\Pr\paren{ \forall j \in [k], \ \abs{Y'_j-\mean} \leq 2 \sqrt{\frac{1}{2^\ell}\ln \frac{2k}{\delta}} } \geq 1-\delta.
		\end{equation}		
		Now, from Theorem 3.1 in \cite{Feldman_COLT_2017}, it follows that the exponential mechanism outputs $\prior$ which is a $\paren{1/4,3/4}$-quantile of $Y'_1,\ldots,Y'_k$ with probability at least $1-\beta$. (In the statement of Theorem 3.1 in \cite{Feldman_COLT_2017}, the condition ``$m \geq 4 \ln(\abs{T}/\beta)/\eps \alpha$" becomes, in our case, $k \geq 16 \ln(\abs{\mathcal{T}}/\beta)/\eps$ after substituting $m = k$, $\alpha = 2$, and accounting for the fact that the cost $c(y)$ has sensitivity $2$ w.r.t. change of a user).
		\medskip
		
		If $\abs{Y'_j-\mean} \leq 2 \sqrt{\frac{1}{2^\ell}\ln \frac{2k}{\delta}}$ holds $\forall j \in [k]$, it must also hold for $\paren{1/4,3/4}$-quantile of $Y'_1,\ldots,Y'_k$. Thus, from \eqref{eq:pvtmed_quantize} and the fact that $\prior$ is a $\paren{1/4,3/4}$-quantile \footnote{For a dataset $s \in \R^n$, a $\paren{1/4,3/4}$-quantile is any $v \in \R$ such that $\abs{\set{i \in [n] : s_i \leq v}} > \frac{n}{4}$ and $\abs{\set{i \in [n] : s_i < v}} < \frac{3n}{4}$.} of $Y'_1,\ldots,Y'_k$ with probability at least $1-\beta$, we get using union bound that
		\[
		\Pr\paren{ \abs{\prior-\mean} \leq 2 \sqrt{\frac{1}{2^\ell}\ln \frac{2k}{\delta}} } \geq 1-\delta-\beta.
		\]	
		
	\end{proof}
	
	\subsection{Proof of Theorem 3.4} \label{supp:mainthm}
	We will prove the following theorem.
	\begin{theorem*} 
		Algorithm \ref{alg:multcounterfull} for continual Bernoulli mean estimation is user-level $\eps$-DP. Moreover, if at time $t \in [T]$, 
		\begin{equation*} 
			\sum_{u=1}^{n}\min\set{m_u(t),\frac{M_t}{2}} \geq \frac{M_t}{2} \frac{16}{\eps} \paren{2L \ln \frac{3L\sqrt{M_t}}{\delta}} \quad (\text{diversity condition})
		\end{equation*}	
		then, with probability at least $1-\delta$,
		\begin{align*}
			\abs{\est_{t} - \mean} = \tilde{O}\paren{ \frac{1}{\sqrt{t}} + \frac{\sqrt{M_t}}{t\eps} }.
		\end{align*}
		Here, $m_u(t)$ is the number of samples obtained from user $u$ till time $t$, and $M_t = \max \set{m_u(t) : u \in [n]}$. 
	\end{theorem*}
	
	\begin{proof}
		Let $L := \lceil \log m \rceil$.
		\paragraph{Privacy.} Algorithm \ref{alg:multcounterfull} uses samples to: 
		\begin{itemize}
			\item[(i)] compute $\prior_\ell$ using ${\rm PrivateMedian}\paren{(x_i,u_i)_{i=1}^t, \frac{\eps}{2L},\ell, \frac{\delta}{3L}}$, for $\ell \in \set{2,\ldots,L}$; and 
			\item[(ii)] compute the array $\binmech[\ell].{\rm NoisyPartialSums}$ for $\ell \in \set{0,\ldots,L}$.
		\end{itemize}
		The output $\paren{\est_{t}}_{t=1}^T$ is computed using $\binmech[\ell].{\rm sum}$, which, in turn is a function of the array \\$\binmech[\ell].{\rm NoisyPartialSums}$, $\ell \in \set{0,\ldots,L}$.
		\medskip
		
		From Claim~\ref{claim:pvtmed}, we get that $\prior_\ell = {\rm PrivateMedian}\paren{(x_i,u_i)_{i=1}^t, \frac{\eps}{2L},\ell, \frac{\delta}{3L}}$ is user-level $\frac{\eps}{2L}$-DP. Thus, from composition property of DP, we get that $(\prior_2,\ldots,\prior_L)$ is user-level $\frac{\eps}{2}$-DP.
		\medskip
		
		Consider $\binmech[\ell]$. A user $u$ contributes at most one element to $\binmech[\ell].{\rm Stream}$; this element is $\Pi_\ell\paren{\sum_{j=2^{\ell-1}+1}^{2^\ell} x^{(u)}_j}$, where $\Pi_\ell(\cdot)$ is the projection on the interval $\cI_\ell$ defined in \ref{eq:fullprojinterval}. 
		So, there are at most $n$ elements in $\binmech[\ell].{\rm Stream}$ throughout the course of the algorithm. 
		Now, a given element in $\binmech[\ell].{\rm Stream}$ will be used at most $(1+\log n)$ times while computing terms in $\binmech[\ell].{\rm NoisyPartialSums}$. Thus, changing a user can change the $\ell_1$-norm of the array $\binmech[\ell].{\rm NoisyPartialSums}$ by at most $(1+\log n) (2\Delta_\ell)$, where $\Delta_\ell$ is as in \eqref{eq:fulldelta}. 
		Hence, adding independent ${\rm Lap}(\eta(m,n,\ell,\delta))$ noise (with $\eta(m,n,\ell,\delta)$ as in \eqref{eq:fullnoise}) while computing each term in $\binmech[\ell].{\rm NoisyPartialSums}$ is sufficient to ensure that the array $\binmech[\ell].{\rm NoisyPartialSums}$ remains user-level $\frac{\eps}{2(L+1)}$-DP throughout the course of the algorithm. Since there are $L+1$ binary mechanisms, by composition property of DP, we get that overall $\paren{\binmech[\ell].{\rm NoisyPartialSums}}_{\ell=0}^L$ is user-level $\frac{\eps}{2}$-DP.
		\medskip
		
		Since $(\prior_2,\ldots,\prior_L)$ is user-level $\frac{\eps}{2}$-DP, and $\paren{\binmech[\ell].{\rm NoisyPartialSums}}_{\ell=0}^L$ is user-level $\frac{\eps}{2}$-DP, we again use composition property to conclude that the output $(\est_t)_{t=1}^T$ by Algorithm~\ref{alg:multcounterfull} is user-level $\eps$-DP.
		
		\paragraph{Utility.} At time $t$, $M_t$ is the maximum number of samples contributed by any user. We will call $\binmech[\ell]$ ``active'' at time $t$, if there are sufficient number of users and samples so that $\prior_\ell$ can be obtained using ${\rm PrivateMedian}\paren{(x_i,u_i)_{i=1}^t, \frac{\eps}{2L},\ell, \frac{\delta}{3L}}$ (Line 10 of Algorithm~\ref{alg:multcounterfull}). Recall that we need  $\prior_\ell$ to create truncation interval $\cI_\ell$ (see \eqref{eq:fullprojinterval}). We know that, for a given user $u$, $x_1^{(u)}$ goes to $\binmech[0].{\rm Stream}$, $x_2^{(u)}$ goes to $\binmech[1].{\rm Stream}$, and $\Pi_\ell\paren{\sum_{j=2^{\ell-1}+1}^{2^\ell} x^{(u)}_j}$ goes to $\binmech[\ell].{\rm Stream}$ for $\ell \geq 2$, provided $\binmech[\ell]$ is ``active''. Thus, at time $t$, since the maximum number of samples contributed by any user is $M_t$, we would like all binary mechanisms till $\binmech[L_t]$ to be ``active'', where $L_t := \lfloor \log M_t \rfloor$. Condition \eqref{eq:diversity} guarantees that we have sufficient number of users and samples to obtain $\prior_2,\ldots,\prior_{L_t}$ via ${\rm PriateMedian}$ (Algorithm~\ref{alg:pvtmed}), thus ensuring that every truncation required at time $t$ is indeed possible. 
		\medskip
		
		\underline{\it All $\prior_\ell$'s are ``good''}: 	
		\smallskip
		
		Suppose diversity condition \eqref{eq:diversity} holds. Then, for each $\ell\in \set{2,\ldots,L}$, ${\rm PrivateMedian}\paren{(x_i,u_i)_{i=1}^t, \frac{\eps}{2L},\ell, \frac{\delta}{3L}}$ outputs a ``good'' $\prior_\ell$ satisfying
		\[
		\abs{\prior_\ell - \mean} \leq 2\sqrt{\frac{1}{2^\ell} \ln\frac{2k(\frac{\eps}{2L},\ell,\frac{\delta}{3L})}{\delta/3L}}
		\]
		with probability at least $1-\frac{2\delta}{3L}$ (see Claim~\ref{claim:pvtmed}). Thus, by union bound, we have the following: with probability at least $1-\frac{2\delta}{3}$, $\prior_\ell$ is ``good'' \textit{for every} $\ell \in \set{2,\ldots,L}$.
		\medskip
		
		
		\underline{\it No truncation happens}:  	
		\smallskip
		
		For a user $u$, for any $\ell \geq 1$, we have from Lemma~\ref{lem:conc_ber} that
		\[
		\Pr\paren{ \abs{\paren{\sum_{j=2^{\ell-1}+1}^{2^\ell}x_j^{(u)}} - (2^{\ell-1})\mean} \leq \sqrt{\frac{2^{\ell-1}}{2} \ln \frac{2n \log m}{\delta/3}} } \geq 1 - \frac{\delta}{3n\log m}.
		\]	
		Taking union bound over $n$ users, we have that
		\[
		\Pr\paren{\forall u \in [n], \ \abs{\paren{\sum_{j=2^{\ell-1}+1}^{2^\ell}x_j^{(u)}} - (2^{\ell-1})\mean} \leq \sqrt{\frac{2^{\ell-1}}{2} \ln \frac{2n \log m}{\delta/3}} } \geq 1 - \frac{\delta}{3\log m}.
		\]	
		Now, since all $\prior_\ell$'s are ``good'' with probablity at least $1-\frac{2\delta}{3}$, we get using union bound that
		\[
		\Pr\paren{\forall u \in [n], \forall \ell \in [L], \ \abs{\paren{\sum_{j=2^{\ell-1}+1}^{2^\ell}x_j^{(u)}} - (2^{\ell-1})\prior_\ell} \leq \sqrt{\frac{2^{\ell-1}}{2} \ln \frac{2n \log m}{\delta}} + \sqrt{2^\ell \ln\frac{2k(\frac{\eps}{2L},\ell,\frac{\delta}{3L})}{\delta/3L}} } \geq 1 - \delta.
		\]
		Note that the projection operator $\Pi_{\ell}(\cdot)$ was defined as projection on interval $\cI_\ell$ as in \eqref{eq:fullprojinterval}. The above equation shows that with probability at least $1-\delta$, the projection operators do not play any role throughout the algorithm, and thus no truncation happens.
		\medskip
		
		\underline{\it Utility at time $t$ ignoring projections $\Pi_\ell$}:  	
		\smallskip
		
		For now, consider Algorithm~\ref{alg:multcounterfull} without the projection operator $\Pi_{\ell}(\cdot)$ in Line 22. Call it Algorithm~\ref{alg:multcounterfull}-NP (NP $\equiv$ `No Projection'). We will derive utility at time $t$ for Algorithm~\ref{alg:multcounterfull}-NP. 
		\medskip
		
		Note that Algorithm~\ref{alg:multcounterfull}~{-NP} does not require $\prior_\ell$'s. Thus, utility of Algorithm~\ref{alg:multcounterfull}-NP would be same as the utility of Algorithm~\ref{alg:multcounter}~{-NP} in the proof of Theorem~3.2 (Section~\ref{subsec:3.2}).
		Hence, at time $t$, for Algorithm~\ref{alg:multcounterfull}-NP,
		\[
		\Pr\paren{ \abs{\est_t - \mean} \leq \sqrt{\frac{1}{t} \ln\frac{2}{\delta}} + \frac{c}{t} \ln\frac{1}{\delta} \sqrt{\sum_{\ell=0}^{\lfloor \log M_t \rfloor}(1+\log n)\eta(m,n,\ell,\delta)^2} } \geq 1 - 2\delta
		\]
		where $\eta(m,n,\ell,\delta)$ is as in \eqref{eq:fullnoise}.
		\medskip
			
		\underline{\it Final utility guarantee for Algorithm~\ref{alg:multcounterfull} at time $t$}:  	
		\smallskip
		
		The above utility guarantee was obtained for Algorithm~\ref{alg:multcounterfull}-NP, which is a variant of Algorithm~\ref{alg:multcounterfull} without projection operator $\Pi_\ell(\cdot)$ in Line 22. But, as argued above, with probability at least $1-\delta$, no truncation happens. Thus, proceeding as in the proof of Theorem~3.1 (Section~\ref{subsec:3.1}), we take a union bound, and get that, with probability at least $1-3\delta$, Algorithm~\ref{alg:multcounterfull} satisfies
		\begin{align*}
			\abs{\est_t - \mean} 
			&\leq \sqrt{\frac{1}{t} \ln\frac{2}{\delta}} + \frac{c}{t} \ln\frac{1}{\delta} \sqrt{\sum_{\ell=0}^{\lfloor \log M_t \rfloor}(1+\log n)\eta(m,n,\ell,\delta)^2} \\
			&= O\paren{ \sqrt{\frac{1}{t} \ln\frac{1}{\delta}} + \frac{\eta(m,n,\lfloor \log M_t \rfloor,\delta)}{t}\sqrt{\log n} \ln\frac{1}{\delta} } \\
			&= O\paren{ \sqrt{\frac{1}{t} \ln\frac{1}{\delta}} + \frac{ \Delta_{\lfloor \log M_t \rfloor}  }{t \eps}(\log m)(\log n)^{3/2} \ln\frac{1}{\delta} } \tag{from \eqref{eq:fullnoise}}\\
			&= O\paren{ \sqrt{\frac{1}{t} \ln\frac{1}{\delta}} + \frac{\sqrt{M_t}}{t\eps} (\log m) (\log n)^{3/2} \paren{\sqrt{\ln\frac{n\log m}{\delta}} + \sqrt{\ln\paren{\frac{(\log m)^2}{\eps \delta}\ln\frac{\sqrt{M_t}}{\delta}}}} \ln\frac{1}{\delta} } \tag{substituting for $\Delta_{\lfloor \log M_t \rfloor}$ from \eqref{eq:fullprojinterval}} \\
			&= \tilde{O}\paren{ \frac{1}{\sqrt{t}} + \frac{\sqrt{M_t}}{t\eps} }.
		\end{align*}
	\end{proof}


\end{document}